\documentclass{opt2020} 

\usepackage{hyperref}
\usepackage{url}
\usepackage{amsmath}
\usepackage{amssymb}
\usepackage{color}
\usepackage{graphicx}
\usepackage{multirow}
\usepackage{epsf}
\usepackage{wrapfig}
\usepackage{bm}
\usepackage{bbm}
\usepackage{enumitem}

\title[SGB: Stochastic Gradient Bound Method]{SGB: Stochastic Gradient Bound Method for Optimizing\\ Partition Functions}

\vspace{-0.1in}

\optauthor{\Name{Jing Wang} \Email{jw5665@nyu.edu}\\
  \Name{Anna Choromanska} \Email{ac5455@nyu.edu}\\
  \addr Department of Electrical and Computer Engineering, New York University}

\makeatletter
\newtheorem*{rep@theorem}{\rep@title}
\newcommand{\newreptheorem}[2]{%
\newenvironment{rep#1}[1]{%
 \def\rep@title{#2 \ref{##1}}%
 \begin{rep@theorem}}%
 {\end{rep@theorem}}}
\makeatother

\newtheorem{thm}{Theorem}
\newtheorem{lma}{Lemma}

\newtheorem{asm}{Assumption}
\newreptheorem{thm}{Theorem}

\newcommand{\vect}[1]{\boldsymbol{\mathbf{#1}}}
\newcommand{\norm}[1]{\left\lVert#1\right\rVert}
\newcommand{\func}[1]{\operatorname{#1}}
\numberwithin{equation}{section}

\begin{document}

\maketitle

\vspace{-0.35in}
\begin{abstract}%
This paper addresses the problem of optimizing partition functions in a stochastic learning setting. We propose a stochastic variant of the bound majorization algorithm from~\cite{jebara2012majorization} that relies on upper-bounding the partition function with a quadratic surrogate. The update of the proposed method, that we refer to as Stochastic Partition Function Bound (SPFB), resembles scaled stochastic gradient descent where the scaling factor relies on a second order term that is however different from the Hessian. Similarly to quasi-Newton schemes, this term is constructed using the stochastic approximation of the value of the function and its gradient. We prove sub-linear convergence rate of the proposed method and show the construction of its low-rank variant (LSPFB). Experiments on logistic regression demonstrate that the proposed schemes significantly outperform SGD. We also discuss how to use quadratic partition function bound for efficient training of deep learning models and in non-convex optimization.
\end{abstract}


\vspace{-0.05in}
\section{Introduction}
\vspace{-0.05in}
The problem of estimating the probability density function over a set of random variables underlies majority of learning frameworks and heavily depends on the partition function. Partition function is a normalizer of a density function and ensures that it integrates to $1$. This function needs to be minimized when learning proper data distribution. Optimizing the partition function however is a hard and often intractable problem~\cite{Goodfellow-et-al-2016}. It has been addressed in a number of ways in the literature. Below we review strategies that directly confront the partition function (we skip pseudo-likelihood strategies~\cite{article} and score matching~\cite{10.5555/1046920.1088696} and ratio matching~\cite{RePEc:eee:csdana:v:51:y:2007:i:5:p:2499-2512} techniques, which avoid direct partition function computations). 

There exists a variety of Markov chain Monte Carlo methods for approximately maximizing the likelihood of models with partition functions such as i) contrastive divergence~\cite{Hinton2002a, CarreiraPerpin2005OnCD} and persistent contrastive divergence~\cite{10.1145/1390156.1390290}, which
perform Gibbs sampling and are used inside a gradient descent procedure to compute model parameter update, and ii) fast persistent contrastive divergence~\cite{10.1145/1553374.1553506}, which relies on re-parametrizing the model and introducing the parameters that are trained with much larger learning rate such that the Markov chain is forced to mix rapidly. The above mentioned techniques are Gibbs sampling strategies that focus on estimating the gradient of the log-partition function. Another technique called noise-contrastive estimation~\cite{10.5555/2188385.2188396} treats partition function like an additional model parameter whose estimate can be learned via nonlinear logistic regression discriminating between the observed data and some artificially generated noise. Methods that directly estimate the partiton function rely on importance sampling. More conretely they estimate the ratio of the partition functions of two models, where one of the partition function is known. The extensions of this technique, annealed importance sampling~\cite{Neal2001,PhysRevLett.78.2690} and bridge sampling~\cite{1976JCoPh..22..245B}, cope with the setting where two considered distributions are far from each other.

Finally, bound majorization constitutes yet another strategy for performing density estimation. Bound majorization methods iteratively construct and optimize variational bound on the original optimization problem. Among these techniques we have iterative scaling schemes~\cite{darroch1972generalized, berger1997improved}, EM algorithm~\cite{balakrishnan2017,Dempster77maximumlikelihood}, non-negative matrix factorization method~\cite{lee99}, convex-concave procedure~\cite{NIPS2001_2125}, minimization by incremental surrogate optimization~\cite{10.1137/140957639}, and technique based on constructing quadratic partition function bound~\cite{jebara2012majorization} (early predecessors of these techniques include~\cite{bohning1988monotonicity, lafferty2001conditional}). The latter technique uses tighter bound compared to the aforementioned methods, and exhibits faster convergence compared to generic first-~\cite{malouf2002comparison, wallach2002efficient, sha2003shallow} and second-order~\cite{zhu1997algorithm,benson2001limited, andrew2007scalable} techniques in the batch optimization setting for both convex and non-convex learning problems. In this paper we revisit the quadratic bound majorization technique and propose its stochastic variant that we analyze both theoretically and empirically. We prove its convergence rate and show that it is performing favorably compared to SGD \cite{robbins1951stochastic, bottou2010large}. Finally, we propose future research directions that can utilize quadratic partition function bound in non-convex optimization, including deep learning setting. With a pressing need to develop landscape-driven deep learning optimization strategies~\cite{feng2020neural,DBLP:journals/corr/ChaudhariCSL17,DBLP:journals/corr/ChaudhariCSL17journal,doi:10.1162/neco.1997.9.1.1,JKABFBS2018ICLRW,DBLP:journals/corr/abs-1805-07898,DBLP:journals/corr/KeskarMNST16,DBLP:journals/corr/SagunEGDB17,DBLP:journals/corr/abs-1901-06053,jiang2019fantastic,DBLP:journals/corr/ChaudhariCSL17,baldassi2015subdominant,baldassi2016unreasonable,baldassi2016multilevel}, we foresee the resurgence of interest in bound majorization techniques and its applicability to non-convex learning problems.

\vspace{-0.15in}
\section{Method}
\vspace{-0.05in}

\begin{figure}[b!]
\vspace{-0.15in}
\begin{minipage}[t]{0.47\textwidth}
\vspace{-1.32in}
\RestyleAlgo{boxruled}
\IncMargin{1em}
\begin{algorithm2e}[H]
\caption{Partition function bound}
\LinesNumbered
\label{alg_PFB}

\KwIn{$\widetilde{\vect{\theta}}\in\mathbb{R}^d$, observation $\vect{x}$, $\vect{f}_{\vect{x}}(y) \forall y\in\{1,...,n\}$.}
\KwOut{Bound parameters: $\vect{\Sigma}$, $\vect{\mu}$, $z$}

\textbf{Init} $z\to 0^+$, $\vect{\mu}=\vect{0}$, $\vect{\Sigma}=z\vect{I}$\\
\For {$y=1,...,n$}{
$\alpha_j=\exp(\widetilde{{\vect{\theta}}}^T\vect{f}_{\vect{x}}(y));\vect{l}=\vect{f}_{\vect{x}}(y)-\vect{\mu}$\\
$\beta=\frac{\tanh(\frac{1}{2}\log(\alpha/z))}{2\log(\alpha/z)}; \kappa=\frac{\alpha}{z+\alpha}$\\
$\vect{\Sigma}+=\beta \vect{l}\vect{l}^T$\\
$\vect{\mu}+=\kappa \vect{l}$\\
$z+=\alpha$}
\end{algorithm2e}
\end{minipage}
\begin{minipage}{0.53\textwidth}
\RestyleAlgo{boxruled}
\IncMargin{1em}
\begin{algorithm2e}[H]
\caption{Maximum Likelihood via Stochastic Partition Function Bound (SPFB)}
\label{alg_SPFB}
\LinesNumbered
\KwIn{initial parameters $\vect{\theta}_0$, training data set $\{(\vect{x}_j,y_j)\}_{j=1}^T$, features $\vect{f}_{\vect{x}_j}$, learning rates $\eta_t$, regularization coefficient $\lambda$}
\KwOut{Model parameters $\vect{\theta}$}

Set $\vect{\theta}=\vect{\theta}_0$\\
\While{not converged}{
    randomly select a training point $(\vect{x}_t,y_t)$\\
    Get $\vect{\Sigma}_t$, $\vect{\mu}_t$ from $\vect{f}_{\vect{x}_t}$, $\vect{\theta}$ via Algorithm \ref{alg_PFB}\\
    $\vect{\theta}\!\leftarrow\!\vect{\theta}\!-\!\eta_t\left(\vect{\Sigma}_t\!+\!\lambda\vect{I}\right)^{-1}\left(\vect{\mu}_t\!-\!\vect{f}_{\vect{x}_t}(y_t)\!+\!\lambda\vect{\theta}\right)$
}
\end{algorithm2e}
\end{minipage}
\vspace{-0.15in}
\end{figure}

Consider the log-linear model given by a density function of the form
\vspace{-0.07in}
\begin{equation}
  \func{p}(y|\vect{x},\vect{\theta})=\exp(\vect{\theta}^T\vect{f_x}(y))/Z_{\vect{x}}(\vect{\theta}),
\end{equation}
\vspace{-0.21in}

\noindent where $(\vect{x},y)$ is the observation-label pair ($y\in\{1,2,...,n\}$), $\vect{f_x}:\{1,2,...,n\}\rightarrow\mathbb{R}^d$ represents a feature map, $\vect{\theta}\in \mathbb{R}^d$ is a model parameter vector, and $Z_{\vect{x}}(\vect{\theta})=\sum_{y=1}^n{\exp(\vect{\theta}^T\vect{f_x}(y))}$ is the partition function. Maximum likelihood framework estimates $\vect{\theta}$ from a training data set $\{(\vect{x}_i,y_i)\}_{i=1}^T$ by maximizing the objective function of the form
\vspace{-0.11in}
\begin{equation}
  J(\vect{\theta})=\sum_{i=1}^T\log \func{p}(y_i|\vect{x}_i,\vect{\theta})-\frac{\lambda}{2}\norm{\vect{\theta}}_2^2 =\sum_{i=1}^T\left[\vect{\theta}^T\vect{f}_{\vect{x}_i}(y)-\log Z_{\vect{x}_i}(\vect{\theta})\right]-\frac{\lambda}{2}\norm{\vect{\theta}}_2^2,
  \label{eq:MaxL}
\end{equation}
\vspace{-0.17in}

\noindent where the second term is a regularization ($\lambda$ is a regularization coefficient). This framework and its various extensions underlie logistic regression, conditional random fields, maximum entropy estimation, latent likelihood, deep belief networks, and other density estimation approaches. Equation~\ref{eq:MaxL} requires minimizing the partition function $Z_{\vect{x}}(\vect{\theta})$. This can be done by optimizing the variational quadratic bound on the partition function instead. The bound is shown in Theorem~\ref{thm:bound}.

\begin{thm}{\rm\cite{jebara2012majorization}}\label{thm_PFB}
  Let $Z_{\vect{x}}(\vect{\theta})=\sum_{y=1}^n\exp\left(\vect{\theta}^T\vect{f}_{\vect{x}}(y)\right)$. Algorithm \ref{alg_PFB} finds $z,\vect{\mu},\vect{\Sigma}$ such that
  \vspace{-0.15in}
  \begin{equation}
  Z_{\vect{x}}(\vect{\theta})\leq z\exp\left(\frac{1}{2}(\vect{\theta}-\vect{\widetilde{\theta}})^T\vect{\Sigma}(\vect{\theta}-\vect{\widetilde{\theta}})+(\vect{\theta}-\vect{\widetilde{\theta}})^T\vect{\mu}\right)
  \end{equation}
  \vspace{-0.2in}
  
  \noindent for any $\vect{\theta},\vect{\widetilde{\theta}},\vect{f}_{\vect{x}}(y)\in\mathbb{R}^d$ for any $y\in\{1,...,n\}$.
  \label{thm:bound}
\end{thm}

\subsection{Stochastic Partition Function Bound (SPFB)}
The partition function bound of Algorithm~\ref{alg_PFB} can be used to  optimize the objective in Equation~\ref{eq:MaxL}. The maximum likelihood parameter update given by the bound takes the form:
\vspace{-0.1in}
\begin{equation}\label{batchmethodupdate}
    \vect{\theta}^{t+1}=\vect{\theta}^{t}-\eta\left(\sum_{j=1}^T\vect{\Sigma}_j+\lambda\vect{I} \right)^{-1}\left(\sum_{j=1}^T\left[\vect{\mu}_j-\vect{f}_{\vect{x}_j}(y_j)\right]+\lambda\vect{\theta}^{t}\right),
\end{equation}
\vspace{-0.15in}

\noindent where $\vect{\Sigma}_j$s and $\vect{\mu}_j$s are computed from Algorithm~\ref{alg_PFB}. In contrast to the above full-batch update, Stochastic Partition Function Bound (SPFB) method that we propose in Algorithm~\ref{alg_SPFB} updates parameters after seeing each training data point, rather than the entire data set, according to the formula:
\vspace{-0.15in}
\begin{equation}\label{SBM_update}
    \vect{\theta}^{t+1}=\vect{\theta}^{t}-\eta_t\left(\vect{\Sigma}_t+\lambda\vect{I}\right)^{-1}\left(\vect{\mu}_t-\vect{f}_{\vect{x}_t}(y_t)+\lambda\vect{\theta}^{t}\right),
\end{equation}
\vspace{-0.2in}

\noindent where $\eta_t=\eta_0/t$ is the learning rate. Denote $f(\vect{\theta};\vect{x}_t)=\log(Z_t(\vect{\theta}))-\vect{\theta}^T\vect{f}_{\vect{x}_t}(\vect{y}_t)+\frac{\lambda}{2}\|\vect{\theta}\|^2$ to be an unbiased estimation of objective function $L(\vect{\theta})$, where $L(\vect{\theta}) = -J(\vect{\theta})$. The above formula (\ref{SBM_update}) can be rewritten as
\vspace{-0.05in}
\begin{equation}\label{SBM_update2}
    \vect{\theta}^{t+1}=\vect{\theta}^t-\eta_t\left(\vect{\Sigma}_t+\lambda\vect{I}\right)^{-1}\nabla f(\vect{\theta}^t;\vect{x}_t).
\end{equation}
\vspace{-0.2in}

\noindent The next theorem shows the convergence rate of SPFB.

\begin{thm}\label{SPFBsublinear}
  $\{\vect{\theta}^t\}$ is the sequence of parameters generated by Algorithm \ref{alg_SPFB}. There exists $0<\mu_1<\mu_2$, $0<\lambda_1<\lambda_2$ such that for all iterations $t$,
  \vspace{-0.1in}
  \begin{equation}
      \mu_1\vect{I}\prec(\vect{\Sigma}_t+\lambda\vect{I})^{-1}\prec\mu_2\vect{I}\quad\text{ and }\quad \lambda_1\vect{I}\prec\nabla^2L(\vect{\theta})\prec\lambda_2\vect{I},
  \end{equation}
  \vspace{-0.3in}
  
  \noindent and there exists a constant $\sigma$, such that for all $\vect{\theta}\in\mathbb{R}^d$, $\vect{E}_{\vect{x}_t}[\norm{f(\vect{\theta};\vect{x}_t)}]^2\leq\sigma^2$. Define the learning rate in iteration $t$ as $\eta_t=\eta_0/t$, where $\eta_0>1/(2\mu_1\lambda_1)$.
  Then for all $t>1$,
    \vspace{-0.1in}
  \begin{equation}
      \vect{E}[L(\vect{\theta}^t)-L(\vect{\theta}^*)]\leq Q(\eta_0)/t,
  \end{equation}
    \vspace{-0.25in}
    
  \noindent where $Q(\eta_0)=\max \left\{\frac{\lambda_2\mu_2^2\eta_0^2\sigma^2}{2(2\mu_1\lambda_1\eta_0-1)},L(\vect{\theta}^1)-L(\vect{\theta^*})\right\}$.
\end{thm}

Theorem \ref{SPFBsublinear} guarantees sub-linear convergence rate for SPFB when the step size is diminishing. However, the time complexity of SPFB is $O(nd^2\!+\!d^3)\!=\!\widetilde{O}(d^3)$, due to the computation and inversion of matrix $\vect{\Sigma}_t$, which is less appealing than the $O(nd)$ complexity of SGD. This is next addressed.

\vspace{-0.1in}
\subsection{Low-rank bound}
In this section, we provide a low-rank construction of the bound that applies to both batch and stochastic setting. We decompose matrix $\vect{\Sigma}$ into $\vect{\Sigma}\!=\!\vect{V}^T\vect{S}\vect{V}\!+\!\vect{D}$, where $\vect{V}\!\in\!\mathbb{R}^{k\times d}$ (orthnormal matrix), $\vect{S}\in\mathbb{R}^{k\times k}$, and $\vect{D}\in\mathbb{R}^{d\times d}$ (diagonal matrix) and apply Woodbury formula to compute the inverse: $\vect{\Sigma}^{-1}=\vect{D}^{-1}-\vect{D}^{-1}\vect{V}^T(\vect{S}^{-1}+\vect{V}\vect{D}^{-1}\vect{V}^T)^{-1}\vect{V}\vect{D}^{-1}$ (clearly, the inverse only requires $O(k^3)$ time and does not affect the total time complexity when rank $k\!\ll\!d$). Note that Algorithm \ref{alg_PFB} performs rank-one update to matrix $\vect{\Sigma}$ of the form: $\vect{\Sigma} = \vect{\Sigma} + \vect{r}\vect{r}^T$, where $\vect{r}\!=\!\sqrt{\beta}\vect{l}$). This update can be ``projected'' onto matrices $\vect{V}, \vect{S}$, and $\vect{D}$. The concrete updates of matrices $\vect{V}, \vect{S}$, and $\vect{D}$ are shown in Algorithm~\ref{alg_LPFB}. The next theorem, Theorem \ref{thm_LR}, guarantees that the low-rank bound is indeed an upper-bound on the partition function \footnote{We simultaneously repair the low-rank bound construction of~\cite{jebara2012majorization}, which breaks this property.}.

\begin{thm}\label{thm_LR}
 Let $Z_{\vect{x}}(\vect{\theta})=\sum_{y=1}^n\exp\left(\vect{\theta}^T\vect{f}_{\vect{x}}(y)\right)$. In each iteration of the $x$-loop in Algorithm \ref{alg_LPFB} finds $z,\vect{\mu},\vect{V},\vect{S},\vect{D}$ such that
  \vspace{-0.15in}
  \begin{equation}
  Z_{\vect{x}}(\vect{\theta})\leq z\exp\left(\frac{1}{2}(\vect{\theta}-\vect{\widetilde{\theta}})^T(\vect{V}^T\vect{S}\vect{V} + \vect{D})(\vect{\theta}-\vect{\widetilde{\theta}})+(\vect{\theta}-\vect{\widetilde{\theta}})^T\vect{\mu}\right)
  \end{equation}
  \vspace{-0.2in}

  \noindent for any $\vect{\theta},\vect{\widetilde{\theta}},\vect{f}_{\vect{x}}(y)\in\mathbb{R}^d$ for any $y\in\{1,...,n\}$.

\end{thm}

Low-rank variant of Algorithm~\ref{alg_SPFB} is presented in Algorithm~\ref{alg_LSPFB}. Note that all proofs supporting this section are deferred to the Supplement.

\noindent\begin{minipage}{1\textwidth}
\RestyleAlgo{boxruled}
\IncMargin{2em}
\begin{algorithm2e}[H]
\caption{Low-rank Partition Function Bound}
\label{alg_LPFB}
\LinesNumbered
\KwIn{$\widetilde{\vect{\theta}}\in\mathbb{R}^d$, observation $\vect{x}$, $\vect{f}_{\vect{x}}(y) \forall y\in\{1,...,n\}$, rank $k\in\mathbb{N}$}
\KwOut{Low-rank bound parameters: $\vect{V}$, $\vect{S}$, $\vect{D}$, $\vect{\mu}$, $z$}
  $z\to 0^+$, $\vect{S}=0$, $\vect{V}=orthonormal\in\mathbb{R}^{k\times d}$, $\vect{D}=z\vect{I}$, $\vect{\mu}=\vect{0}$\\
  \For(\tcp*[h]{$x$-loop}){each sample $\vect{x}_j$ in batch}{ 
    Init $z_j\gets 0^+$, $\vect{\upsilon}=0$\\
    \For{each label $y\in\{1,2,...,n\}$}{
      $\alpha=\exp(\widetilde{\vect{\theta}}^T\vect{f}_{\vect{x}_j}(y))$; 
  $\vect{r}=\sqrt{\frac{\tanh(\frac{1}{2}\log(\alpha/z))}{2\log(\alpha/z)}}(\vect{f}_{\vect{x}_j}(y)-\vect{\upsilon})$;\\
  
    
      $\vect{p}=\vect{V}\vect{r}$; $\vect{a}=\vect{V}^T\vect{p}$; $\vect{g}=\vect{r}-\vect{a}$, $\vect{S}+\!=\vect{p}\vect{p}^T$\\
      $\vect{Q}^T\vect{A}\vect{Q}=svd(\vect{S})$; $\vect{S}\gets \vect{A}$; $\vect{V}\gets \vect{QV}$; $\vect{D}+\!=\|\vect{g}\|\|\vect{a}\|\vect{I}\in\mathbb{R}^{d\times d}$\\
      $\vect{s}=[\vect{S}(1,1),...,\vect{S}(k,k),\|\vect{g}\|^2]^T$, $\widetilde{k}=\arg\min_{i=1,...,k+1}\vect{s}(i)$\\
      \uIf{$\widetilde{k}\leq k$}{
        $\vect{D}=\vect{D}+\vect{S}(\widetilde{k},\widetilde{k})1^T|\vect{V}(\widetilde{k},\cdot)|diag(|\vect{V}(\widetilde{k},\cdot)|)$\\
        $\vect{S}(\widetilde{k},\widetilde{k})=\|\vect{g}\|^2$; $\vect{g}=\frac{\vect{g}}{\|\vect{g}\|}$; $\vect{V}(\widetilde{k},\cdot)=\vect{g}$
      }
      \uElse{
        $\vect{D}+\!=1^T|\vect{g}|diag(|\vect{g}|)$
      }
      $\vect{\upsilon}+\!=\frac{\alpha}{z_j+\alpha}(\vect{f}_{\vect{x}_j}(y)-\vect{\upsilon})$; $z_j+\!=\alpha$\\
    }
    $\vect{\mu}+\!=\vect{\upsilon}$, $z+\!=z_j$\\
  }
\end{algorithm2e}
\end{minipage}\\

\noindent\begin{minipage}{1\textwidth}
\RestyleAlgo{boxruled}
\IncMargin{1em}
\begin{algorithm2e}[H]
\caption{MLE via Low-rank Stochastic Partition Function Bound (LSPFB)}
\label{alg_LSPFB}
\LinesNumbered
\KwIn{initial parameters $\vect{\theta}_0$, training data set $\{(\vect{x}_j,y_j)\}_{j=1}^T$, features $\vect{f}_{\vect{x}_j}$, learning rates $\eta_t$, regularization coefficient $\lambda$}
\KwOut{Model parameters $\vect{\theta}$}

Set $\vect{\theta}=\vect{\theta}_0$\\
\While{not converged}{
    randomly select a training point $(\vect{x}_t,y_t)$\\
    Get $\vect{V}_t$, $\vect{S}_t$, $\vect{D}_t$, $\vect{\mu}_t$ from $\vect{x}_t$, $\vect{f}_{\vect{x}_t}$, $\vect{\theta}$ via Algorithm \ref{alg_LPFB} (input batch is a single data point $\vect{x}_t$)\\
    $\vect{D}_t=\vect{D}_t+\lambda\vect{I}$; $\vect{\mu}_t=\vect{\mu}_t-\vect{f}_{\vect{x}_t}(y_t)+\lambda\vect{\theta}$\\
    $\vect{\theta}\!\leftarrow\!\vect{\theta}\!-\!\eta_t\left(\vect{D}_t^{-1}-\vect{D}_t^{-1}\vect{V}_t^T\left(\vect{S}_t^{-1}+\vect{V}_t\vect{D}_t^{-1}\vect{V}_t^T\right)^{-1}\vect{V}_t\vect{D}_t^{-1}\right)\vect{\mu}_t$
}
\end{algorithm2e}
\end{minipage}

\vspace{-0.15in}
\section{Experiments}
\vspace{-0.05in}

Experiments were performed on adult\footnote[1]{http://archive.ics.uci.edu/ml/datasets/Adult} ($T=48842$, $n=2$, $d=14$), KMNIST\footnote[2]{https://pytorch.org/docs/stable/torchvision/datasets.html\label{web}} ($T=60000$, $n=10$, $d=784$), and Fashion-MNIST\textsuperscript{\ref{web}} ($T=60000$, $n=10$, $d=784$) data sets. All algorithms were run with mini-batch size equal to $m=1000$. For low-rank methods, we explored the following settings of the rank $k$: $k=1, 5, 10, 100$. Hyperparameters for all methods were chosen to achieve the best test performance. We compare SPFB and LSPFB with SGD for $\ell_2$-regularized logistic regression on the adult, KMNIST, and Fashion-MNIST data sets. Both SPFB and LSPFB show clear advantage over SGD in terms of convergence speed.

\begin{figure}[h]
\vspace{-0.1in}
\floatconts
  {fig}
  {\vspace{-0.35in}\caption{A Comparison of SPFB, LSPFB, and SGD on $l2$-regularized logistic regression problem.}}
  {%
    \subfigure{
      \includegraphics[width=0.24\linewidth]{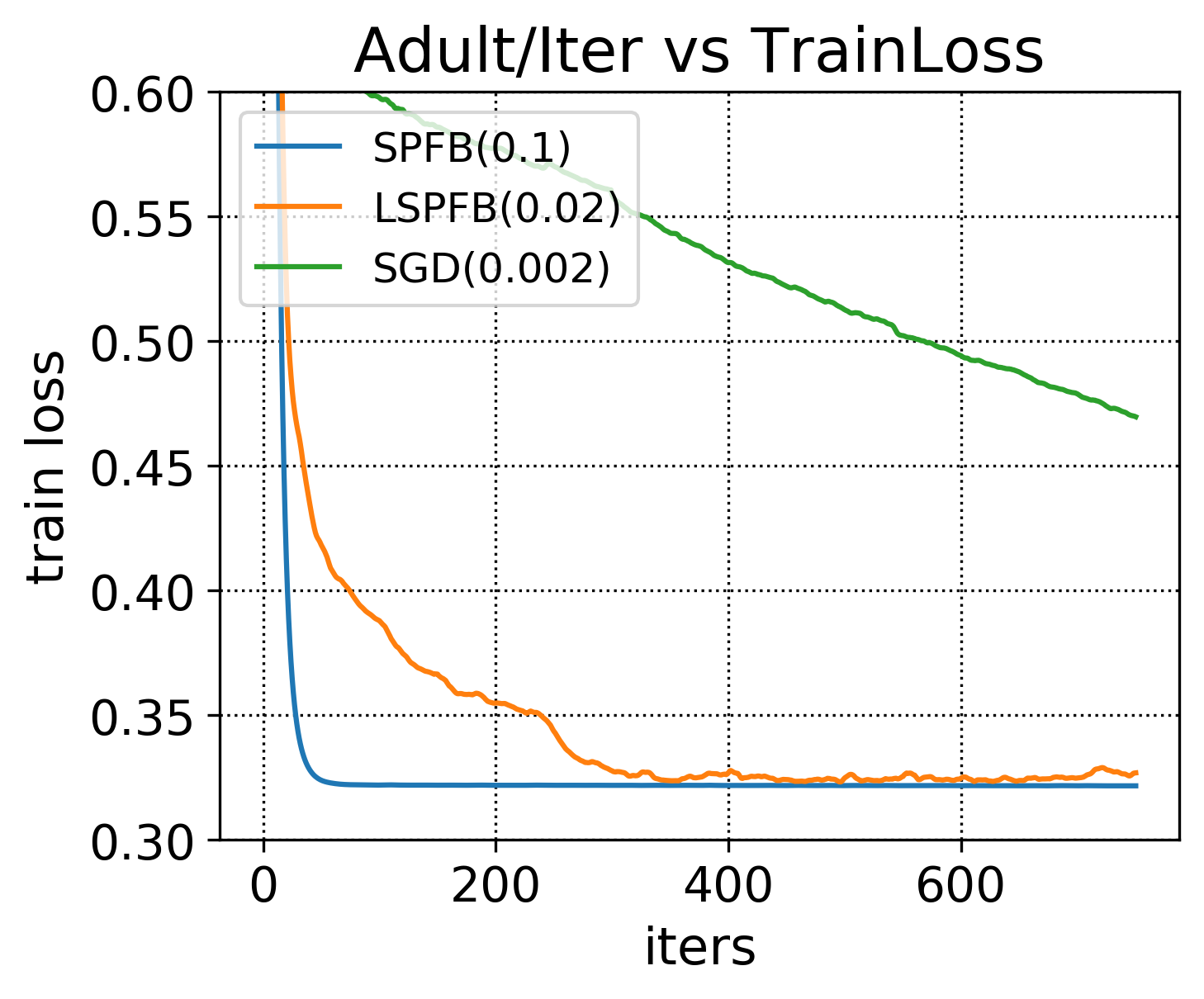}}%
    \hfill
    \subfigure{
      \includegraphics[width=0.23\linewidth]{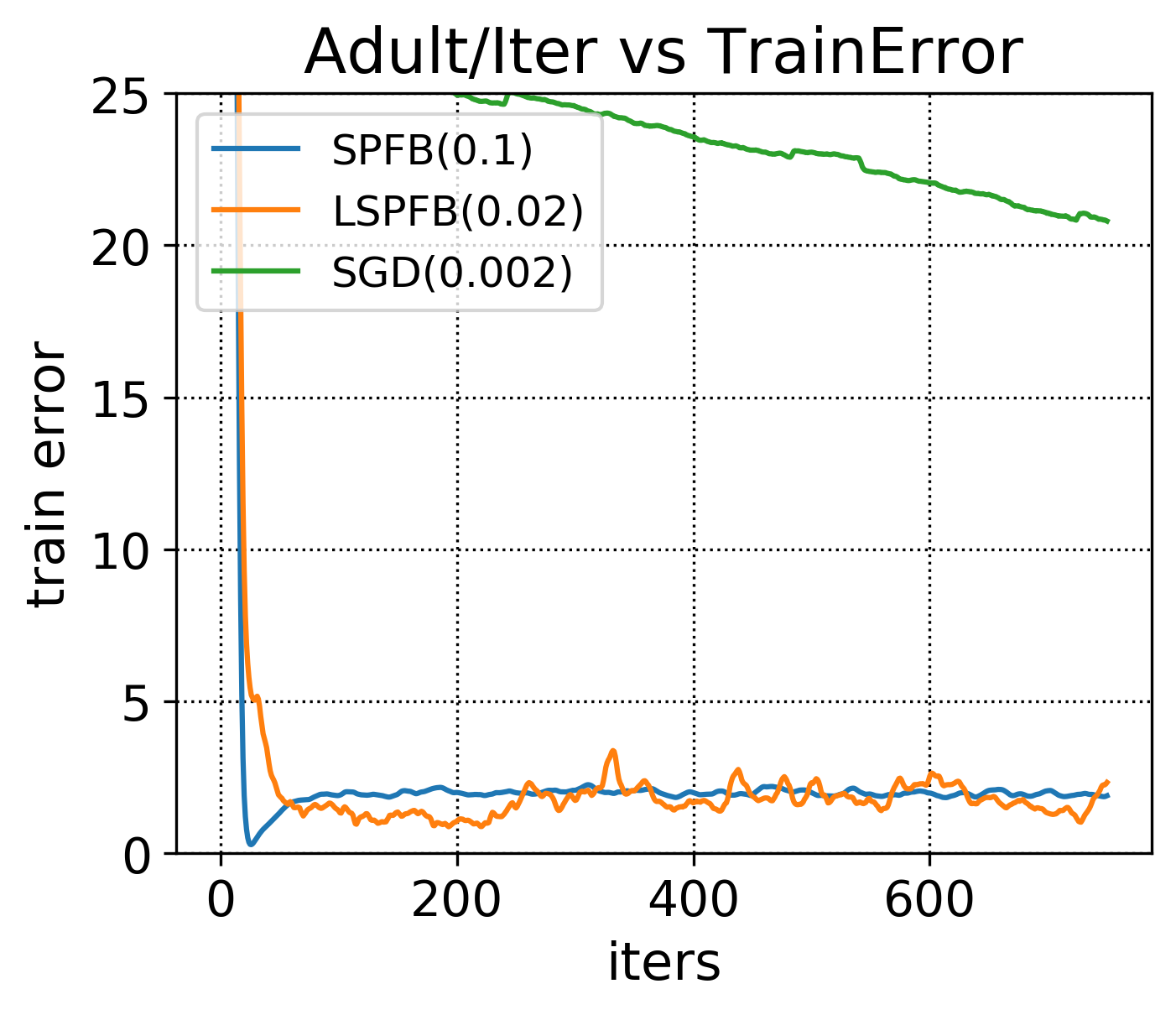}}%
     \hfill
     \subfigure{
      \includegraphics[width=0.24\linewidth]{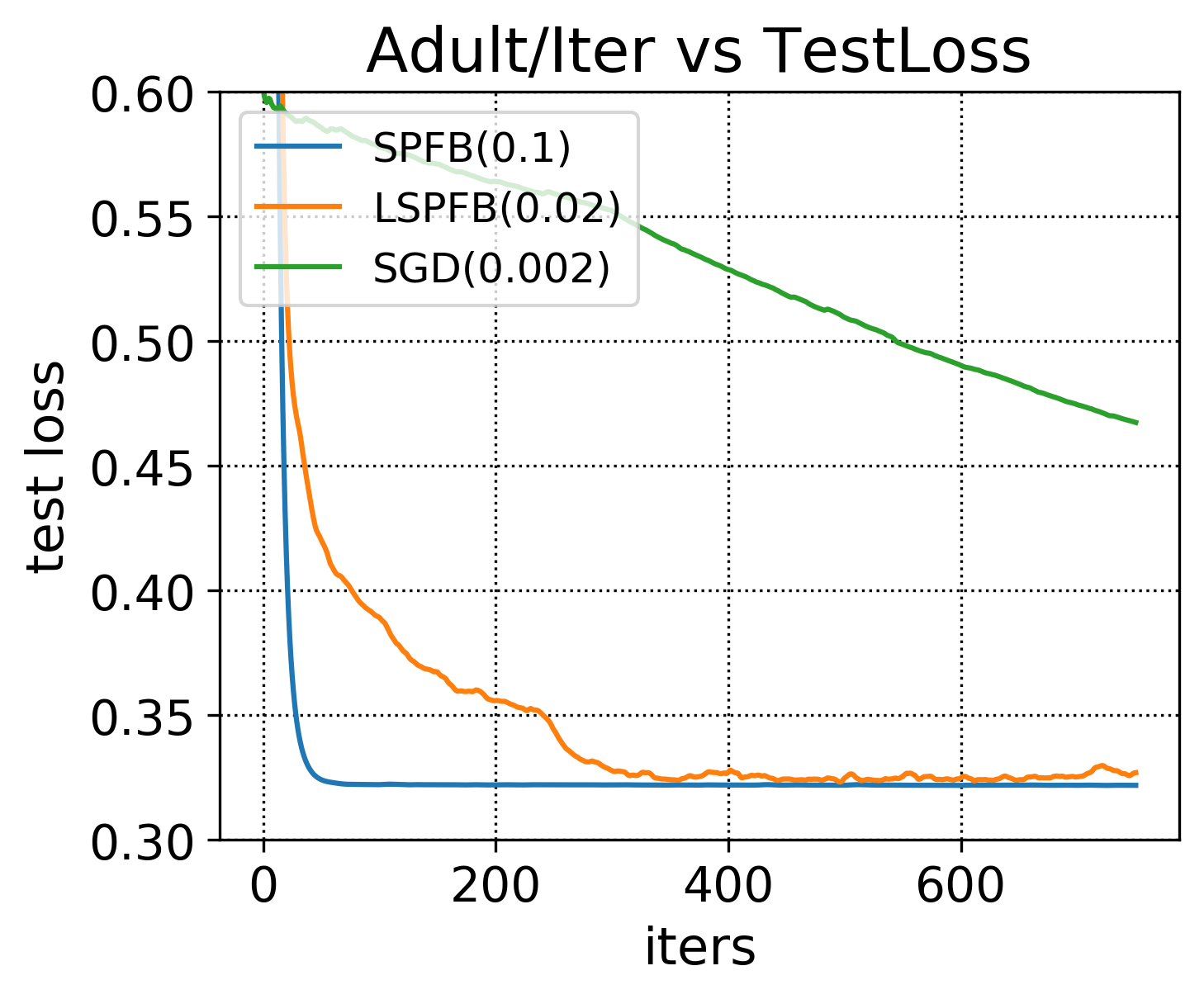}}%
    \hfill
    \subfigure{
      \includegraphics[width=0.23\linewidth]{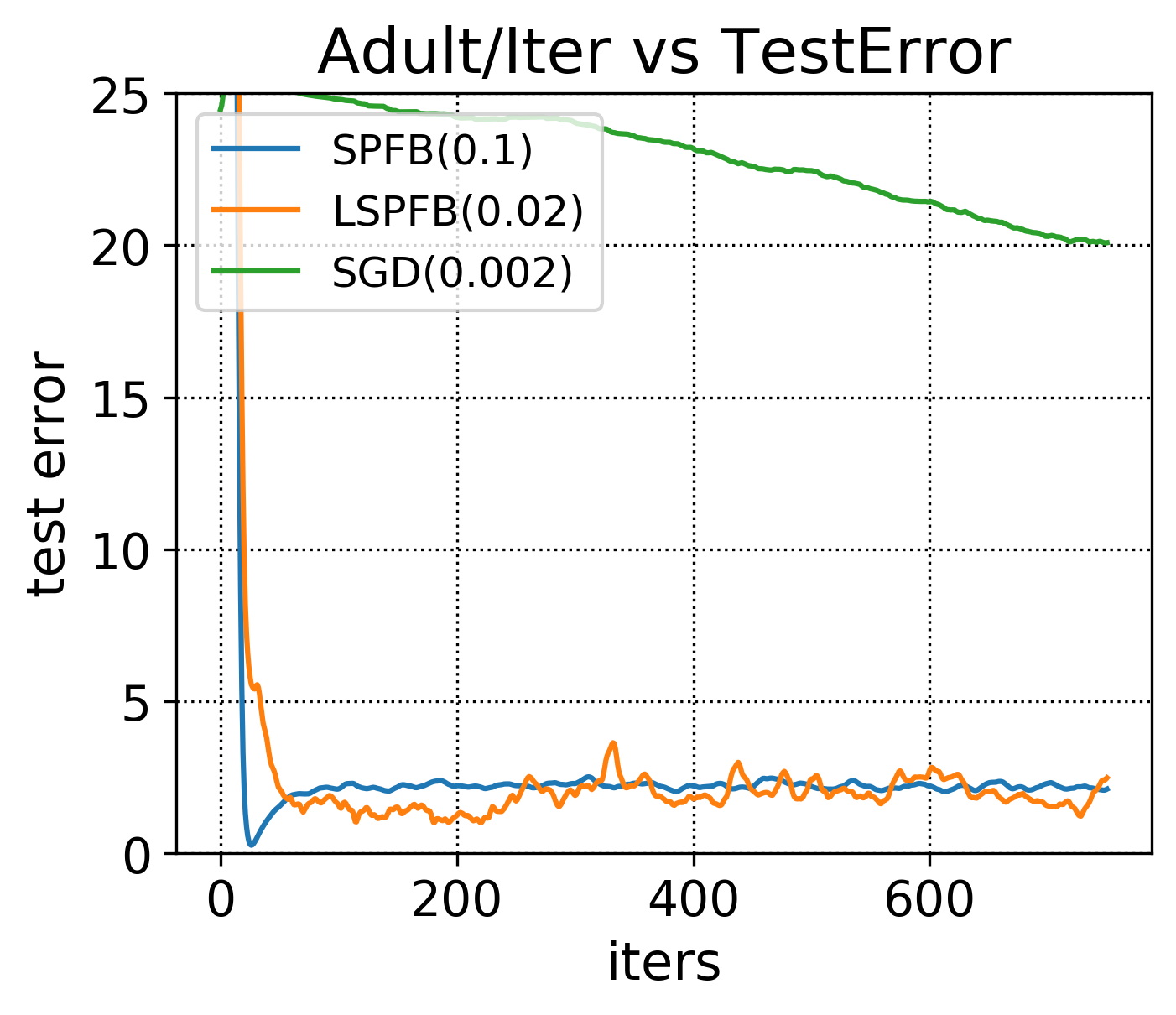}}
     \\
     \subfigure{
      \includegraphics[width=0.23\linewidth]{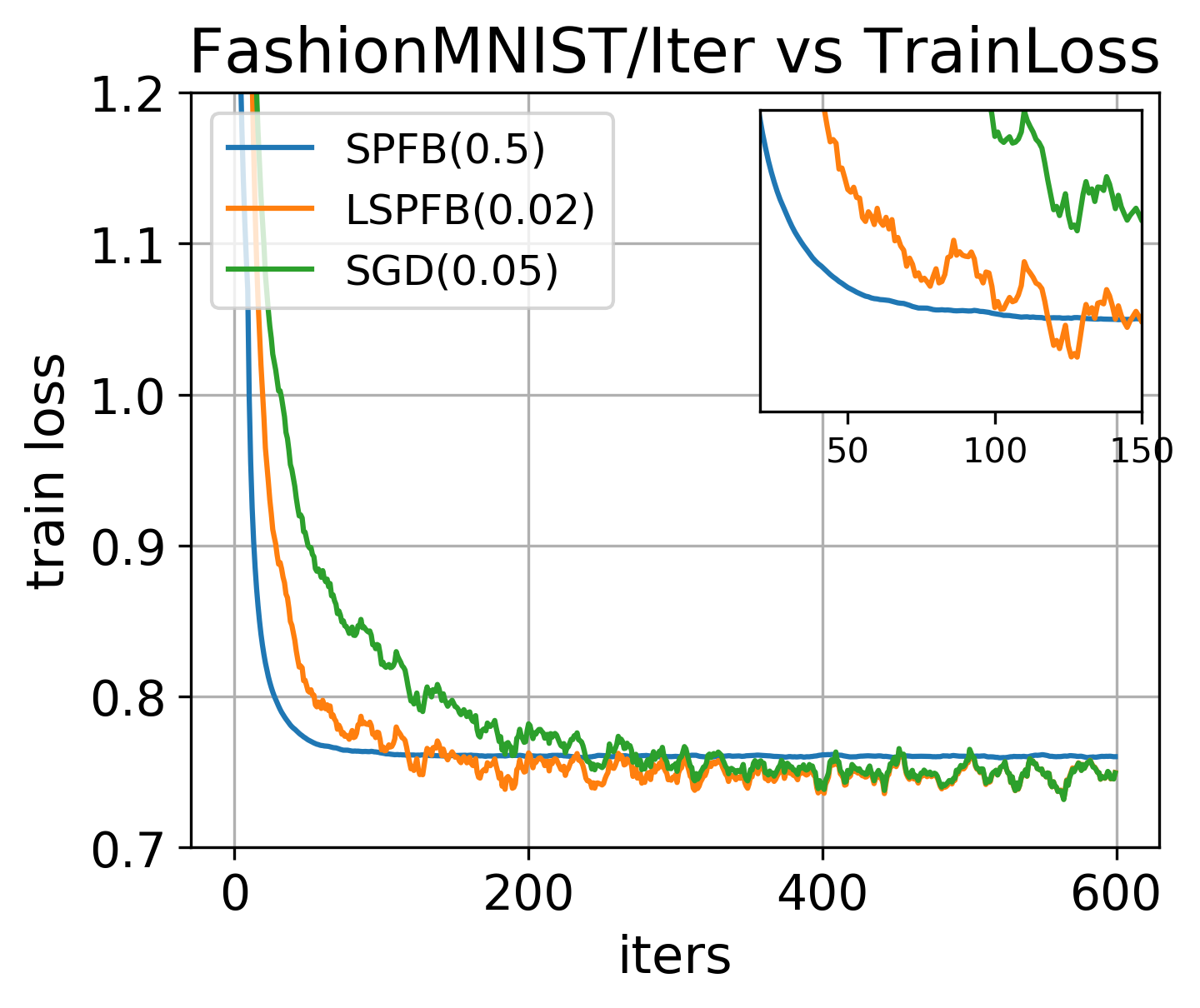}}%
    \hfill
    \subfigure{
      \includegraphics[width=0.23\linewidth]{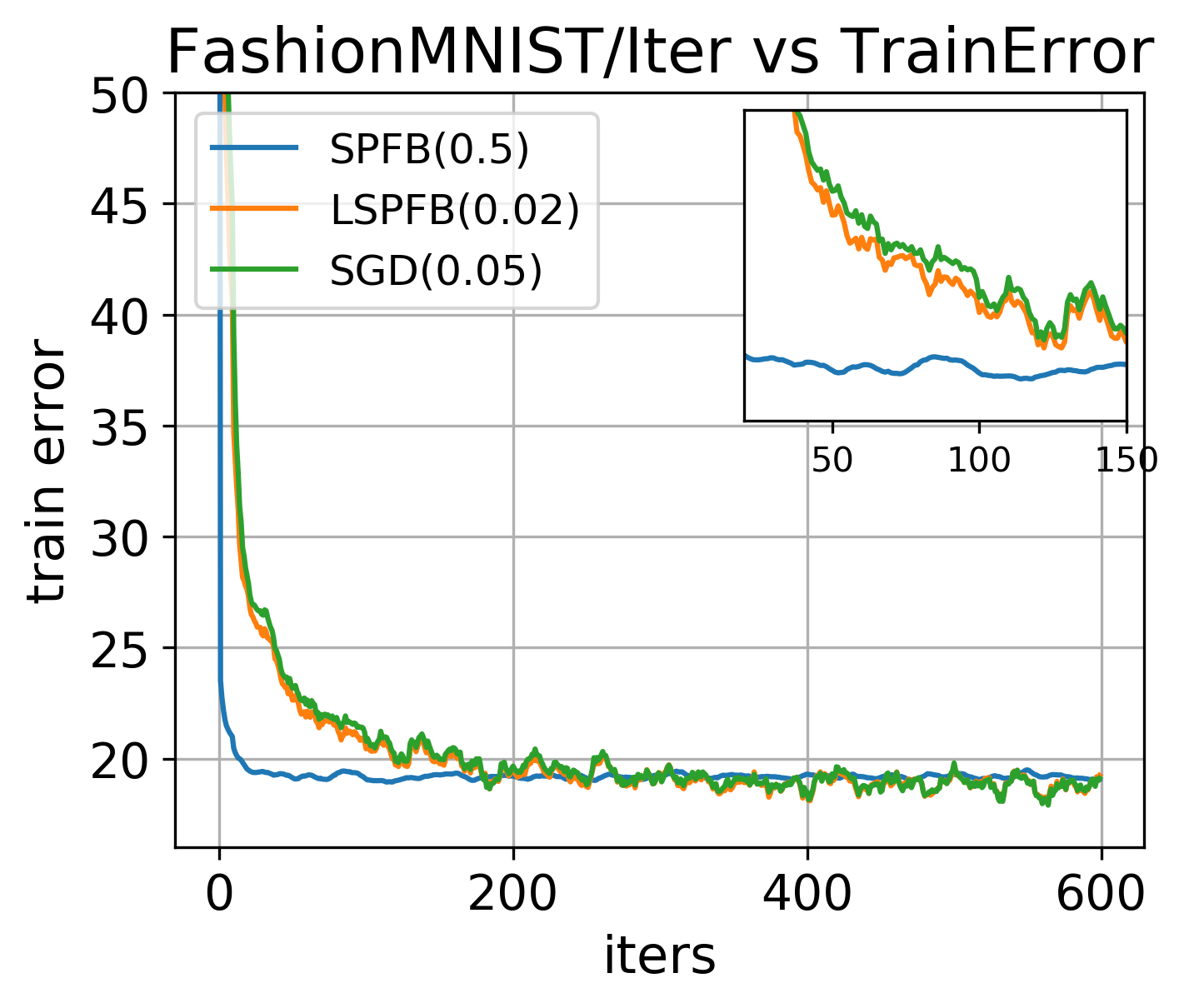}}%
     \hfill
     \subfigure{
      \includegraphics[width=0.23\linewidth]{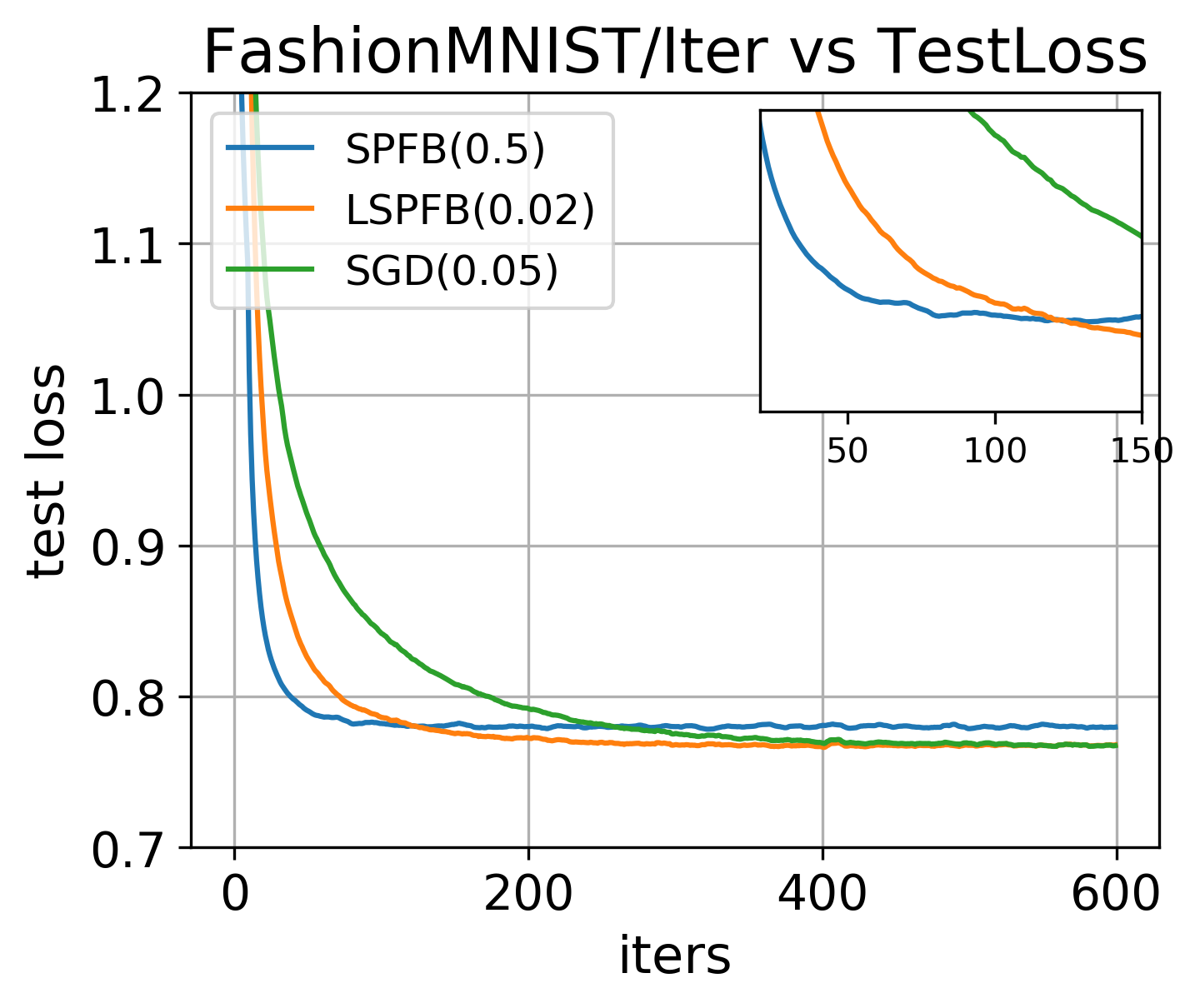}}%
    \hfill
    \subfigure{
      \includegraphics[width=0.23\linewidth]{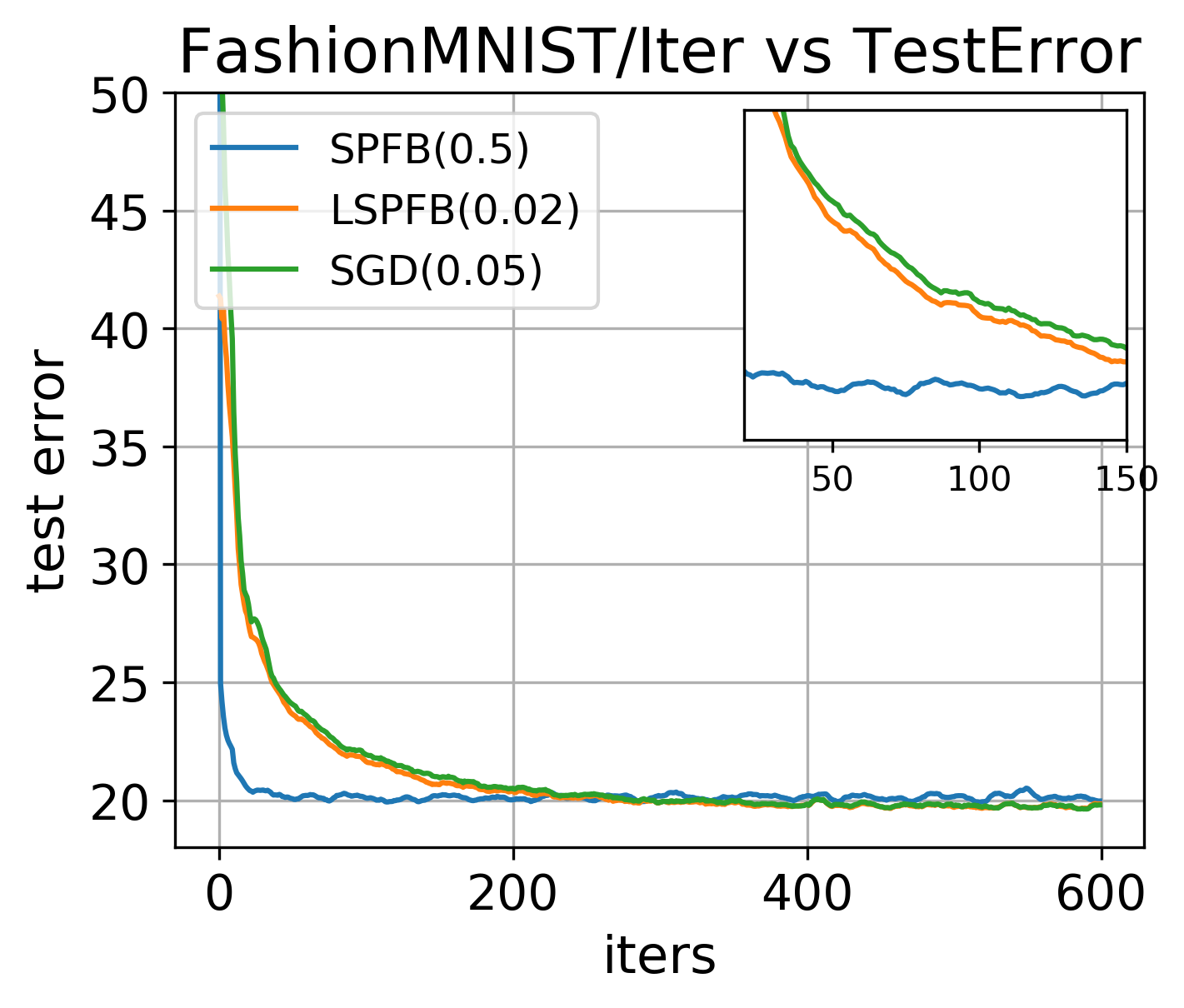}}
    \\
    \subfigure{
      \includegraphics[width=0.23\linewidth]{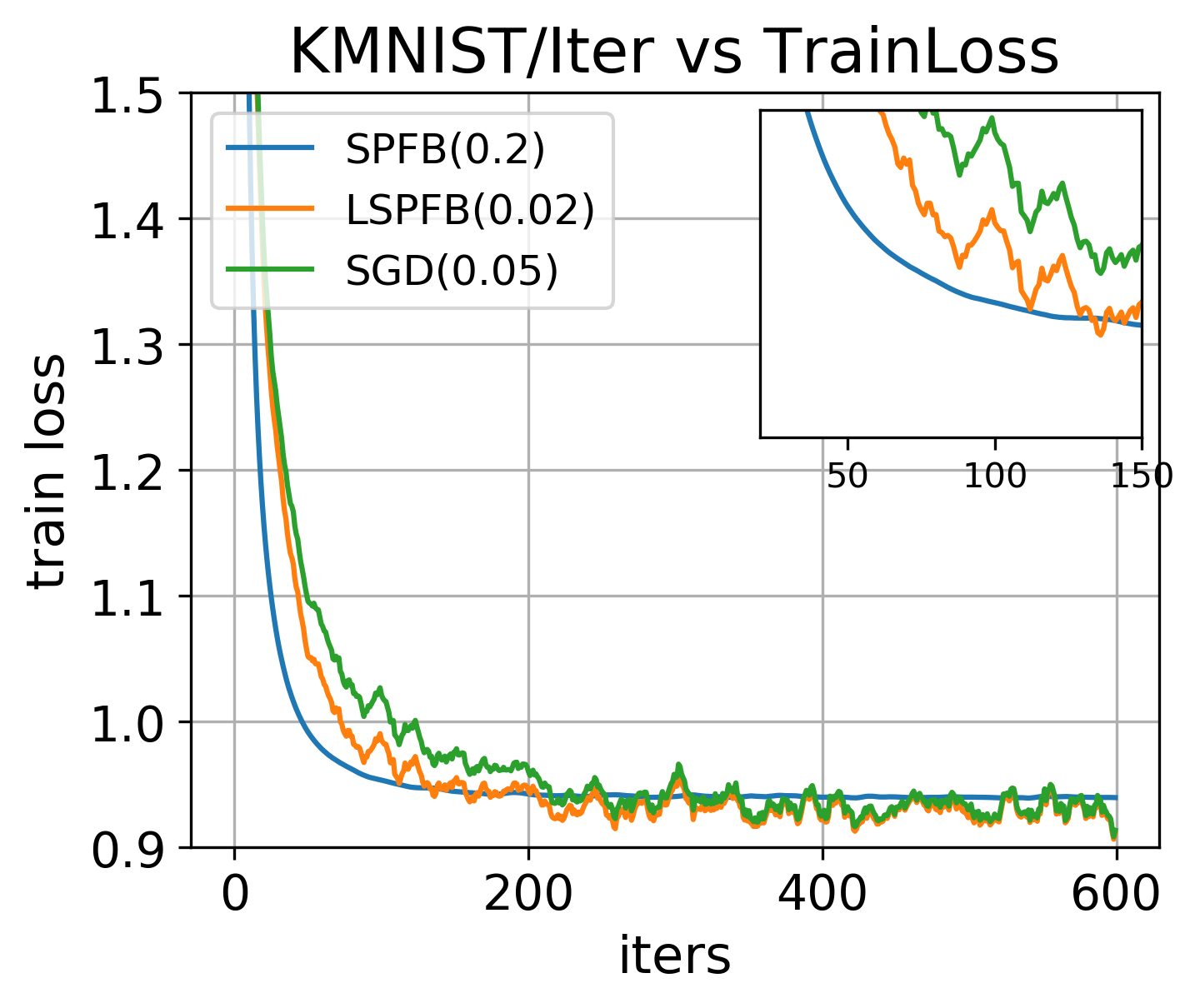}}%
    \hfill
    \subfigure{
      \includegraphics[width=0.23\linewidth]{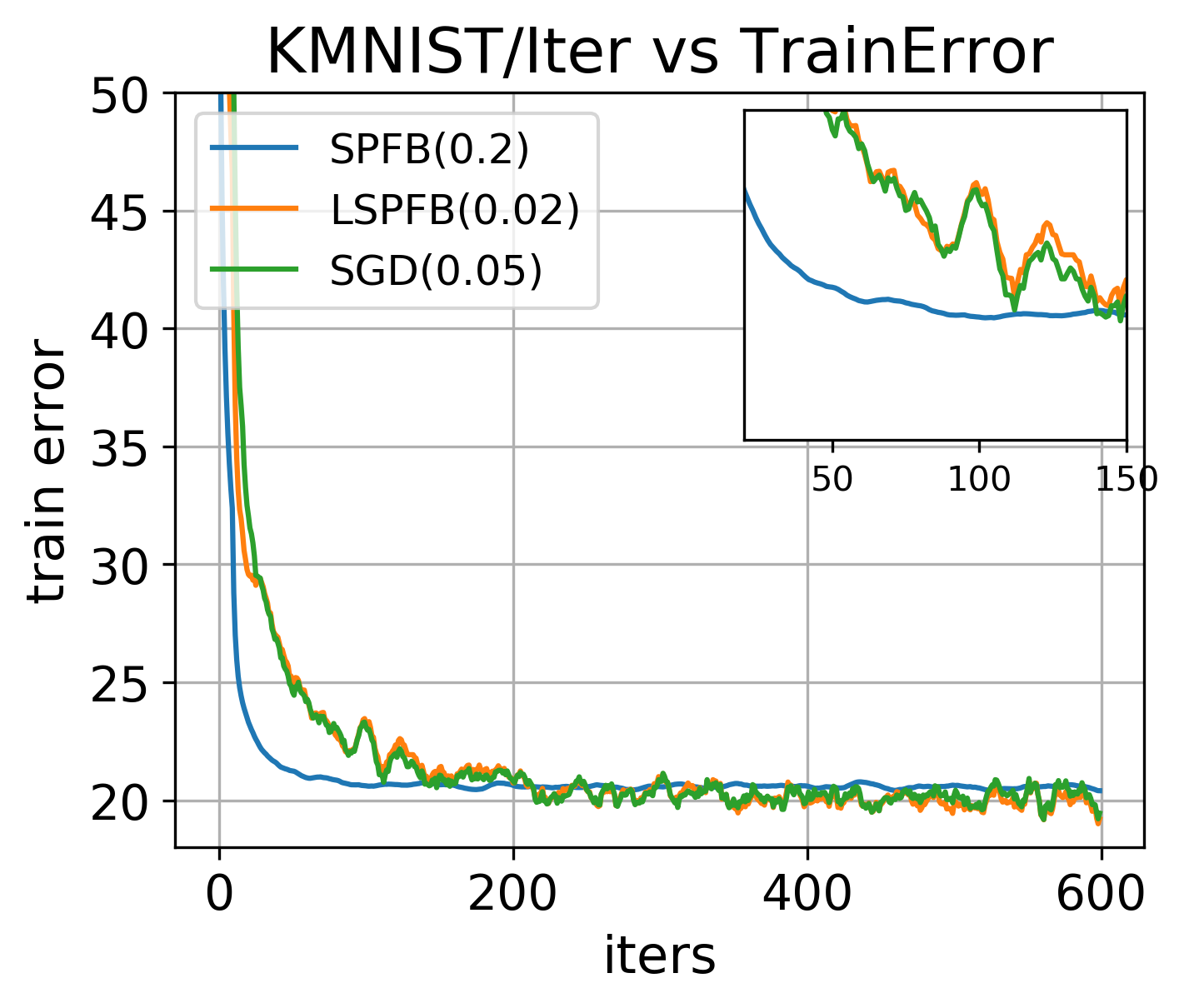}}%
     \hfill
     \subfigure{
      \includegraphics[width=0.23\linewidth]{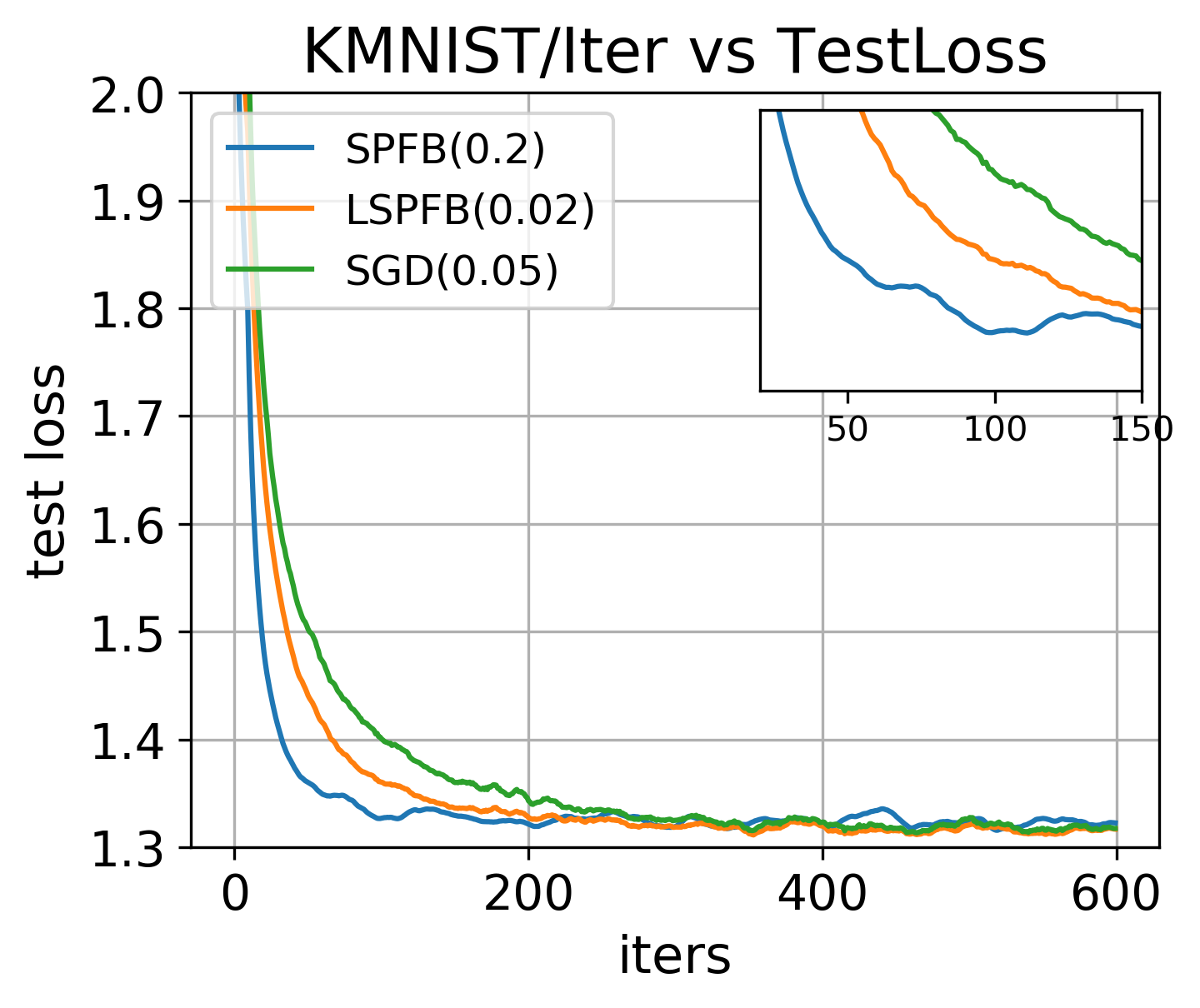}}%
    \hfill
    \subfigure{
      \includegraphics[width=0.23\linewidth]{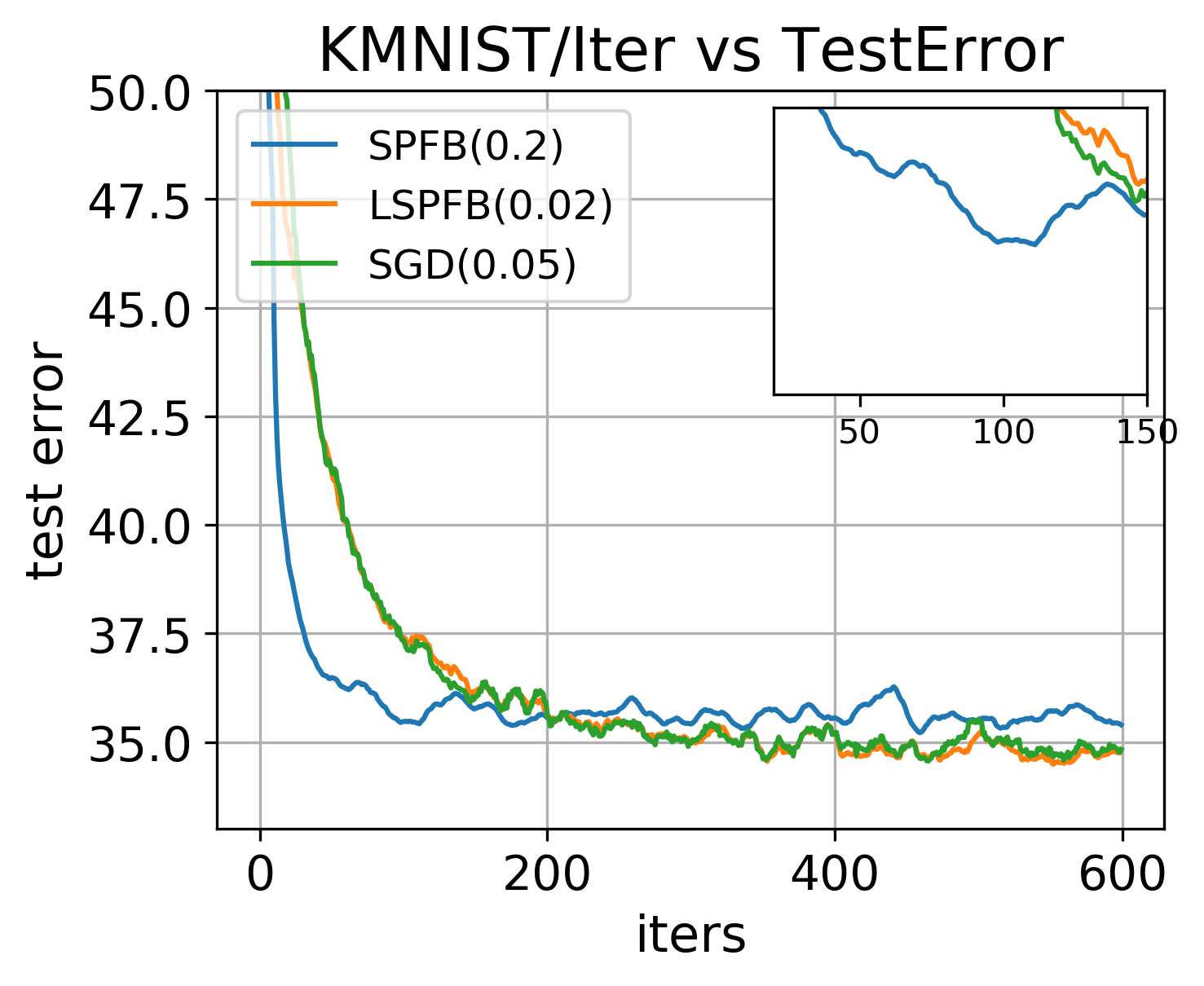}}
  }
  \vspace{-0.35in}
\end{figure}

\section{Discussion}
\vspace{-0.05in}

Here we briefly discuss the future extensions of this work. First, we will analyze whether the batch bound method of Algorithm~\ref{alg_PFB} admits super-linear convergence rate and thus match the convergence rate of quasi-Newton techniques (the existing convergence analysis in the original paper shows linear rate only). On the empirical side, we will investigate applying the bound techniques discussed in this paper to optimize each layer of the network during backpropagation. Per-layer bounds, when combined together, may potentially lead to a universal quadratic bound on the original highly complex deep learning loss function. This approach would open up new possibilities for training deep learning models as it reduces the deep learning non-convex optimization problem to a convex one. We will also investigate applying the developed technique to backpropagation-free~\cite{DBLP:conf/icml/ChoromanskaCKLR19} setting and large-batch~\cite{LARS} training of deep learning models. Finally, we will explore the applicability of the bound techniques in the context of biasing the gradient to explore wide valleys in the non-convex optimization landscape ~\cite{DBLP:journals/corr/ChaudhariCSL17}. In this case enforcing the width of the bound to be sufficiently large should provide a simple mechanism for finding solutions that lie in the flat regions of the landscape.



\bibliography{main}
\newpage
\clearpage
\appendix

\noindent\rule{\textwidth}{1pt}
{\par \Large \bf  \centering{Supplementary material} \par}
\noindent\rule{\textwidth}{1pt}

\section{Proof for Theorem \ref{SPFBsublinear}: sub-linear convergence rate}
The proof in this section inspires by \cite{broyden1973local, byrd2016stochastic,}. We analysis the convergence rate of Stochastic version of Bound Majorization Method. It is easy to know that the objective function 
$$L(\vect{\theta})=\frac{1}{T}\sum_{t=1}^T\left[\log(Z_{\vect{x}_t}(\vect{\theta}))-\vect{\theta}^T\vect{f}_{\vect{x}_t}(\vect{y}_t)\right]+\frac{\lambda}{2}\norm{\vect{\theta}}^2$$
is strongly convex and and twice continuously differentiable, where $Z_{\vect{x}_t}(\vect{\theta})=\sum_{i=1}^n\exp(\vect{\theta}^T\vect{f}_{\vect{x}_t}(i))$. Therfore, we can make the following assumption

\begin{asm} \label{asm1}
  Assume that
  \begin{itemize}
      \item[(1)] The objective function $L$ is twice continuous differentiable.
      \item[(2)] There exists $0<\lambda_1<\lambda_2$ such that for all $\vect{\theta}\in\mathbb{R}^d$,
            \begin{equation}
                \lambda_1\vect{I}\prec\nabla^2L(\vect{\theta})\prec\lambda_2\vect{I}, 
            \end{equation}
  \end{itemize}
\end{asm}

\noindent For stochastic partition function bound (SPFB) method, define $\eta_t$ to be the learning rate, $\vect{\Sigma}_t$ to be the Hessian approximation and
$f(\vect{\theta};\vect{x}_t)=\log(Z_t(\vect{\theta}))-\vect{\theta}^T\vect{f}_{\vect{x}_t}(\vect{y}_t)+\frac{\lambda}{2}\|\vect{\theta}\|^2$ to be the unbiased estimation of objective function $L(\vect{\theta})$ in each iteration, then the update for parameter $\vect{\theta}$ is
\begin{equation}\label{Fupdate}
    \vect{\theta}^{t+1}=\vect{\theta}^t-\eta_t(\vect{\Sigma}_t+\lambda\vect{I})^{-1}\nabla f(\vect{\theta}^t;\vect{x}_t).
\end{equation}
In order to analysis the convergence for SPFB, we add a condition that the variance of estimation $f(\vect{\theta})$ is bounded then construct the following assumption:
\begin{asm}\label{asm2}
  Assume that 
  \begin{itemize}
      \item[(1)] The objective function $L$ is twice continuous differentiable.
      \item[(2)] There exists $0<\lambda_1<\lambda_2$ such that for all $\vect{\theta}\in\mathbb{R}^d$,
        \begin{equation}
          \lambda_1\vect{I}\prec\nabla^2L(\vect{\theta})\prec\lambda_2\vect{I},
        \end{equation}
      \item[(3)]There exists a constant $\sigma$, such that for all $\vect{\theta}\in\mathbb{R}^d$,
       \begin{equation}
          \vect{E}_{\vect{x}}[\norm{f(\vect{\theta};\vect{x})}]^2\leq\sigma^2
       \end{equation}
  \end{itemize}{}
\end{asm}

\begin{lma} \label{lemmaF1}
For matrices $\{\vect{\Sigma}_t\}$ compute in \textbf{Algorithm} \ref{alg_PFB}, there exists $0<\mu_1<\mu_2$ such that for all $\vect{\Sigma}_t$
\begin{equation}
    \mu_1\vect{I}\prec(\vect{\Sigma}_t+\lambda\vect{I})^{-1}\prec\mu_2\vect{I}.
\end{equation}
\end{lma}

\begin{proof}
For easy notation, we omit the subscripts and demote $\vect{\Sigma}_t$ as $\vect{\Sigma}$. From \textbf{Algorithm}\ref{alg_PFB}, the formulation for $\vect{\Sigma}$ is
\begin{equation}
    \vect{\Sigma}=z_0\vect{I}+\sum_{i=1}^n\beta_i\vect{l}_i\vect{l}_i^T,
\end{equation}

\noindent where $z_0\to 0^+$ is the initialization of $z$ in algorithm \ref{alg_PFB}. Define $z_i=\sum_{k=1}^i\alpha_k$, we can compute the upper bound of $\beta_i$ 
 \begin{equation}
    \beta_i=\frac{\tanh{(\frac{1}{2}\log(\alpha_i/z_i)})}{2\log(\alpha_i/z_i)}
    =\frac{1}{4}\cdot\frac{\tanh{(\frac{1}{2}\log(\alpha_i/z_i))}}{\frac{1}{2}\log(\alpha_i/z_i)}\leq\frac{1}{4},
 \end{equation}
 we can also conclude that there exits $k$ such that $\frac{1}{n}<\frac{\alpha_k}{z_k}<1$, which imply
 $$\beta_k\geq\frac{\tanh(\frac{1}{2}\log(\frac{1}{n}))}{2\log(\frac{1}{n})}.$$
 Since $\vect{l}_{i+1}=(-\frac{\alpha_1}{z_i}\vect{x},\cdots,\frac{\alpha_i}{z_i}\vect{x},\vect{x},0,\cdots,0)$, where $\vect{x}$ is the current observation, we have
 \begin{equation}
    (1+\frac{1}{n})\|\vect{x}\|^2\leq
    \|\vect{l}_{i+1}\|^2=\frac{\alpha_1^2+\alpha_2^2+\cdots\alpha_i^2}{(\alpha_1+\alpha_2+\cdots\alpha_i)^2}\|\vect{x}\|^2+\|\vect{x}\|^2
    \leq2\|\vect{x}\|^2
 \end{equation}
 Therefore, 
 \begin{align}
    \|\vect{\Sigma}\|^2\geq\beta_k\|\vect{l}_k\|^2\geq\left(1+\frac{1}{n}\right)\frac{\tanh(\frac{1}{2}\log(\frac{1}{n}))}{2\log(\frac{1}{n})}\|\vect{x}\|^2.
 \end{align}
 Based on previous proof, the upperbound for matrix $\vect{\Sigma}$ is
 \begin{align}
    \|\vect{\Sigma}\|^2&\leq z_0^2+\sum_{i=1}^n\beta_i\|\vect{l}_i\vect{l}_i^T\|^2
    \leq z_0^2+\frac{1}{4}\sum_{i=1}^n\|\vect{l}_i\|^4\notag\\
    &\leq z_0^2+\frac{1}{4}\|\vect{x}\|^4\sum_{i=1}^n2
    \leq z_0^2+\sqrt{\frac{n}{2}}\|\vect{x}\|^4.
\end{align}
When $z_0\rightarrow 0^+$, $\|\vect{\Sigma}\|\leq \sqrt{\frac{n}{2}}\|\vect{x}\|^2$.

Denote $\{\vect{x}_t\}$ as the set of all observations. Define $\mu_2=1/\left(\left(1+\frac{1}{n}\right)\frac{\tanh(\frac{1}{2}\log(\frac{1}{n}))}{2\log(\frac{1}{n})}\max_t\|\vect{x}_t\|^2+\lambda\right)$,\\ $\mu_1=1/\left(\sqrt{\frac{n}{2}}\max_t\|\vect{x}_t\|^2+\lambda\right)$,
$$1/\mu_2\leq\|\vect{\Sigma}_t+\lambda\vect{I}\|\leq1/\mu_1\Longrightarrow
    \mu_2\geq\|(\vect{\Sigma}_t+\lambda\vect{I})^{-1}\|\geq\mu_1.$$
    
\end{proof}

\begin{repthm}{SPFBsublinear}
  $\{\vect{\theta}^t\}$ is the sequence of parameters generated by Algorithm \ref{alg_SPFB}. There exists $0<\mu_1<\mu_2$, $0<\lambda_1<\lambda_2$ such that for all iterations $t$,
  \begin{equation}
      \mu_1\vect{I}\prec(\vect{\Sigma}_t+\lambda\vect{I})^{-1}\prec\mu_2\vect{I}\quad\text{ and }\quad \lambda_1\vect{I}\prec\nabla^2L(\vect{\theta})\prec\lambda_2\vect{I},
  \end{equation}
  and there exists a constant $\sigma$, such that for all $\vect{\theta}\in\mathbb{R}^d$,
   \begin{equation}
      \vect{E}_{\vect{x}_t}[\norm{f(\vect{\theta};\vect{x}_t)}]^2\leq\sigma^2.
   \end{equation}
  Define the learning rate in iteration $t$ as 
  $$\eta_t=\eta_0/t\quad\quad\quad\text{where } \eta_0>1/(2\mu_1\lambda_1).$$
  Then for all $t>1$,
  \begin{equation}\label{Fconv}
      \vect{E}[L(\vect{\theta}^t)-L(\vect{\theta}^*)]\leq Q(\eta_0)/t,
  \end{equation}
  where $Q(\eta_0)=\max \left\{\frac{\lambda_2\mu_2^2\eta_0^2\sigma^2}{2(2\mu_1\lambda_1\eta_0-1)},L(\vect{\theta}^1)-L(\vect{\theta^*})\right\}$.
\end{repthm}

\begin{proof}
\begin{align}
     L(\vect{\theta}^{t+1})&=L(\vect{\theta}^t-\eta_t(\vect{\Sigma}_t+\lambda\vect{I})^{-1}\nabla f(\vect{\theta}^t;\vect{x}_t))\notag\\
     &\leq L(\vect{\theta}^t)+\nabla L(\vect{\theta}^t)^T\left(-\eta_t(\vect{\Sigma}_t+\lambda\vect{I})^{-1}\nabla f(\vect{\theta}^t;\vect{x}_t)\right)+\frac{\lambda_2}{2}\norm{\eta_t(\vect{\Sigma}_t+\lambda\vect{I})^{-1}\nabla f(\vect{\theta}^t;\vect{x}_t)}^2\notag\\
     &\leq L(\vect{\theta}^t)-\eta_t\nabla L(\vect{\theta}^t)^T(\vect{\Sigma}_t+\lambda\vect{I})^{-1}\nabla f(\vect{\theta}^t;\vect{x}_t)+\frac{\lambda_2}{2}\eta_t^2\mu_2^2\norm{\nabla f(\vect{\theta}^t;\vect{x}_t)}^2.
  \end{align}\label{propF1_1}
  Taking the expectation, we have
  \begin{align}\label{propF1_2}
    \vect{E}[L(\vect{\theta}^{t+1})]\leq\vect{E}[L(\vect{\theta}^t)]-\eta_t\mu_1\norm{\nabla L(\vect{\theta}^t)}^2+\frac{\lambda_2}{2}\eta_t^2\mu_2^2\sigma^2.
  \end{align}
  Now, we want to correlate $\norm{\nabla L(\vect{\theta}^t)}^2$ with $L(\vect{\theta^t})-L(\vect{\theta}^*)$. For all $\vect{v}\in\mathbb{R}^d$, by condition (2) in Assumption \ref{asm2},
  \begin{align}
    L(\vect{v})&\geq L(\vect{\theta}^t)+\nabla L(\vect{\theta}^t)^T(\vect{v}-\vect{\theta}^t)+\frac{\lambda_1}{2}\|\vect{v}-\vect{\theta}^t\|^2\notag\\
    &\geq L(\vect{\theta}^t)-\nabla L(\vect{\theta}^t)^T\left(\frac{1}{\lambda_1}\nabla L(\vect{\theta}^t)\right)+\frac{\lambda_1}{2}\|\frac{1}{\lambda_1}\nabla L(\vect{\theta}^t)\|^2\notag\\
    &= L(\vect{\theta}^t)-\frac{1}{2\lambda_1}\|\nabla L(\vect{\theta}^t)\|^2,
  \end{align}\label{propF1_3}
    the second inequality comes from computing the minimum of quadratic function $q(\vect{v})=L(\vect{\theta}^k)+\nabla L(\vect{\theta}^k)^T(\vect{v}-\vect{\theta}^k)+\frac{\lambda_1}{2}\|\vect{v}-\vect{\theta}^k\|^2$. Setting $\vect{v}=\vect{\theta^*}$ yields
    \begin{equation}
    2\lambda_1[L(\vect{\theta^t})-L(\vect{\theta}^*)]\leq\|\nabla L(\vect{\theta}^t)\|^2,
    \end{equation}
    which together with formula (\ref{propF1_2}) yields
  \begin{align}
    \vect{E}[L(\vect{\theta}^{t+1})-L(\vect{\theta}^*)]&\leq \vect{E}[L(\vect{\theta}^{t})-L(\vect{\theta}^*)]-2\eta_t\mu_1\lambda_1\vect{E}[L(\vect{\theta}^t)-L(\vect{\theta}^*)]+\frac{\lambda_2}{2}\eta_t^2\mu_2^2\sigma^2\notag\\
    &=(1-2\eta_t\mu_1\lambda_1)\vect{E}[L(\vect{\theta}^t)-L(\vect{\theta}^*)]+\frac{\lambda_2}{2}\eta_t^2\mu_2^2\sigma^2.
  \end{align}
  Define $\phi_t=\vect{E}[L(\vect{\theta}^t)-L(\vect{\theta}^*)]$, $Q(\eta_0)=\max \left\{\frac{\lambda_2\mu_2^2\eta_0^2\sigma^2}{2(2\mu_1\lambda_1\eta_0-1)},L(\vect{\theta}^1)-L(\vect{\theta}^*)\right\}$, when $t=1$, $\phi_1=L(\vect{\theta}^1)-L(\vect{\theta}^*)=Q(\eta_0)/1$ holds. We finish the proof by induction. Assume $\phi_t\leq\frac{Q(\eta_0)}{t}$,
  \begin{align}
    \phi_{t+1}&\leq (1-2\eta_t\mu_1\lambda_1)\phi_t+\frac{\lambda_2}{2}\eta_t^2\mu_2^2\sigma^2\notag\\
    &=(1-\frac{2\eta_0\mu_1\lambda_1}{t})\frac{Q(\eta_0)}{t}+\frac{\lambda_2\eta_0^2\mu_2^2\sigma^2}{2t^2}\notag\\
    &= \frac{t-2\eta_0\mu_1\lambda_1}{t^2}Q(\eta_0)+\frac{\lambda_2\eta_0^2\mu_2^2\sigma^2}{2t^2}\notag\\
    &=\frac{t-1}{t^2}Q(\eta_0)-\frac{2\eta_0\mu_1\lambda_1-1}{t^2}Q(\eta_0)+\frac{\lambda_2\eta_0^2\mu_2^2\sigma^2}{2t^2}\notag\\
    &\leq \frac{t-1}{t^2}Q(\eta_0)-\frac{2\eta_0\mu_1\lambda_1-1}{t^2}\frac{\lambda_2\mu_2^2\eta_0^2\sigma^2}{2(2\mu_1\lambda_1\eta_0-1)}+\frac{\lambda_2\eta_0^2\mu_2^2\sigma^2}{2t^2}\notag\\
    &\leq \frac{Q(\eta_0)}{t+1}
  \end{align}
\end{proof}

\section{Proof for Thm \ref{thm_LR}: Low-rank Bound}

The lower-bound in \cite{jebara2012majorization} is not correct since $\vect{a}\perp\vect{g}$ is not sufficient for $\vect{x}^T(\vect{ag}^T+\vect{ga}^T)\vect{x}=0,\forall\vect{x}$. We loose and fix the bound in this section. We use some proof steps in \cite{jebara2012majorization}.

\begin{lma} \label{lemma1}
$\forall \vect{x}\in\mathbb{R}^d$, $\forall \vect{l}\in\mathbb{R}^{+d}(l\geq0)$, we have
$$\sum_{i=1}^d\vect{x}^2(i)\vect{l}(i)\geq\left(\sum_{i=1}^d\vect{x}(i)\frac{\vect{l}(i)}{\sqrt{\sum_{j=1}^d\vect{l}(j)}}\right)$$
\end{lma}
\begin{proof}
By Jensen's inequality, if $f:\mathbb{R}\rightarrow\mathbb{R}$ convex, $\{a_i\}_{i=1}^d$ satisfies $\sum_{i=1}^da_i=1$, then
$$\sum_{i=1}^da_if(\vect{x}(i))\geq f\left(\sum_{i=1}^da_i\vect{x}(i)\right)$$
Set $a_i=\frac{\vect{l}(i)}{\sum_{i=1}^d\vect{l}(i)}$,$f(x)=x^2$,
\begin{align*}
    \sum_{i=1}^d\vect{x}(i)^2\frac{\vect{l}(i)}{\sum_{i=1}^d\vect{l}(i)}&\geq \left(\sum_{i=1}^d\vect{x}(i)\frac{\vect{l}(i)}{\sum_{i=1}^d\vect{l}(i)}\right)^2\\
    \sum_{i=1}^d\vect{x}^2(i)\vect{l}(i)&\geq\left(\sum_{i=1}^d\vect{x}(i)\frac{\vect{l}(i)}{\sqrt{\sum_{j=1}^d\vect{l}(j)}}\right)
\end{align*}
\end{proof}

\begin{lma} \label{lemma2}
If $\vect{a},\vect{g}\in\mathbb{R}^{n}$ non-zero, then $rank(\vect{a}\vect{g}^T)=1$.
\end{lma}
\begin{proof}
 $$rank(\vect{a}\vect{g}^T)=\min\{rank(\vect{a}),rank(\vect{g})\}=1$$
\end{proof}

\begin{lma} \label{lemma3}
Let $\vect{a},\vect{g}\in\mathbb{R}^n$ non-zero, matrix $\vect{A}=\vect{a}\vect{g}^T+\vect{g}\vect{a}^T\in\mathbb{R}^{n\times n}$. If $\vect{a}\perp \vect{g}$ and $\vect{A}\neq\vect{0}$, then \vect{A} has exactly 2 opposite eigenvalues $\pm\|\vect{a}\|_2\|\vect{g}\|_2$. 
\end{lma}
\begin{proof}
 Because $rank(\vect{ag}^T)=rank(\vect{ga}^T)=1$ (By \textbf{Lemma} \ref{lemma2}),
 $$rank(\vect{A})=rank(\vect{ag}^T+\vect{ga}^T)\leq rank(\vect{ag}^T)+rank(\vect{ga}^T)= 2.$$
 
 \begin{itemize}
     \item[(a)] $rank(\vect{A})=1$, there is only one non-zero eigenvalue $\lambda$.
   $$\lambda=trace(\vect{A})=\sum_{i=1}^n\vect{A}_{ii}=\sum_{i=1}^n\vect{a}_i\vect{g}_i+\vect{a}_i\vect{g}_i=2\vect{a}^T\vect{g}=0.$$
   Therefore, all eigenvalues of matrix $\vect{A}$ are 0 and $\vect{A}=\vect{0}$, which contradicts the condition $\vect{A}\neq0$.
    \item[(b)] $rank(\vect{A})=2$, assume $\lambda_1, \lambda_2$ are 2 non-zero eigenvalues of $\vect{A}$.
    \begin{equation}
        \lambda_1+\lambda_2=trace(\vect{A})=0\label{opposite}
    \end{equation}
    Without loss of generality, assume $\lambda_1>0$, then the characteristic polynomial of matrix $A$ is
    $$\lambda^{n-2}(\lambda-\lambda_1)(\lambda-\lambda_2)=\lambda^{n-2}(\lambda^2-\lambda_1^2),$$
    and matrix $A$ satisfies
    \begin{align}
      \vect{A}^{n-2}(\vect{A}^2-\lambda_1^2\vect{I})=\vect{0}\label{eigenpoly1}
    \end{align}
    Because
    \begin{align*}
      \vect{A}^3=\vect{A}^2\vect{A}&=(\vect{ag}^T\vect{ag}^T+\vect{ag}^T\vect{ga}^T+\vect{ga}^T\vect{ag}^T+\vect{ga}^T\vect{ga}^T)(\vect{ag}^T+\vect{ga}^T)\notag\\
      &=(\|\vect{a}\|_2^2\vect{gg}^T+\|\vect{g}\|_2^2\vect{aa}^T)(\vect{ag}^T+\vect{ga}^T)\notag\\
      &=\|\vect{a}\|_2^2\|\vect{g}\|_2^2(\vect{ag}^T+\vect{ga}^T)\notag\\
      &=\|\vect{a}\|_2^2\|\vect{g}\|_2^2\vect{A},\notag\\
    \end{align*}
    We can conclude that
    \begin{equation}
        \vect{A}(\vect{A}^2-\|\vect{a}\|_2^2\|\vect{g}\|_2^2\vect{I})=0\Longrightarrow \vect{A}^{n-2}(\vect{A}^2-\|\vect{a}\|_2^2\|\vect{g}\|_2^2\vect{I})=0.\label{eigenpoly2}
    \end{equation}
    From formula $(\ref{opposite})$, $(\ref{eigenpoly1})$ and $(\ref{eigenpoly2})$ we can conclude that the eigenvalues $\lambda_1=\|\vect{a}\|_2\|\vect{g}\|_2$ and $\lambda_2=-\|\vect{a}\|_2\|\vect{g}\|_2$.
 \end{itemize}
\end{proof}

\begin{repthm}{thm_LR}
 Let $Z_{\vect{x}}(\vect{\theta})=\sum_{y=1}^n\exp\left(\vect{\theta}^T\vect{f}_{\vect{x}}(y)\right)$. In each iteration of the $x$-loop in Algorithm \ref{alg_LPFB} finds $z,\vect{\mu},\vect{V},\vect{S},\vect{D}$ such that
  \vspace{-0.15in}
  \begin{equation}\label{LPFB_1}
  Z_{\vect{x}}(\vect{\theta})\leq z\exp\left(\frac{1}{2}(\vect{\theta}-\vect{\widetilde{\theta}})^T(\vect{V}^T\vect{S}\vect{V} + \vect{D})(\vect{\theta}-\vect{\widetilde{\theta}})+(\vect{\theta}-\vect{\widetilde{\theta}})^T\vect{\mu}\right)
  \end{equation}
  \vspace{-0.2in}
\end{repthm}

\begin{proof}
Define $Z_{\vect{x}}^i(\vect{\theta})=\sum_{y=1}^i\exp(\vect{\theta}^T\vect{f}_{\vect{x}}(y))$, then the partition function is represented as $Z_{\vect{x}}(\vect{\theta})=Z_{\vect{x}}^n(\vect{\theta})$, where n is the number of labels. From proof of Thm \ref{thm_PFB} (details included in \cite{jebara2012majorization}), we already find a sequence of matrix $\{\vect{\Sigma}_i\}_{i=1}^n$, vector $\{\vect{\mu}\}_{i=1}^n$ and constant $\{z_i\}_{i=1}^n$ such that:

\begin{equation}\label{PFB_1}
    Z_{\vect{x}}^i(\vect{\theta})\leq z_i\exp\left(\frac{1}{2}(\vect{\theta}-\vect{\widetilde{\theta}})^T\vect{\Sigma}_i(\vect{\theta}-\vect{\widetilde{\theta}})+(\vect{\theta}-\vect{\widetilde{\theta}})^T\vect{\mu}_i\right),
\end{equation}

\noindent where the terminate terms $\vect{\Sigma}=\vect{\Sigma}_n$, $\vect{\mu}=\vect{\mu}_n$, $z=z_n$ are the output of algorithm $\ref{alg_PFB}$. If we could find sequences of $\{\vect{V}_i\}_{i=1}^n\subset\mathbb{R}^{k\times d}$ (orthnormal), $\{\vect{S}_i\}_{i=1}^n\subset\mathbb{R}^{k\times k}$ and $\{\vect{D}_i\}_{i=1}^n\subset\mathbb{R}^{d\times d}$ (diagonal) upper-bounds matrices $\{\vect{\Sigma}_i\}_{i=1}^n\subset\mathbb{R}^{d\times d}$ as

\begin{equation}\label{sigma_bound}
\vect{x}^T\vect{\Sigma}_i\vect{x}\leq \vect{x}^T(\vect{V}_i^T\vect{S}_i\vect{V}_i+\vect{D}_i)\vect{x}\quad\forall\vect{x}\in\mathbb{R}^d,
\end{equation}
the upper-bound \ref{PFB_1} of partition function over only $i$ labels can be renewed as
\begin{equation}\label{LPFB_2}
    Z_{\vect{x}}^i(\vect{\theta})\leq z_i\exp\left(\frac{1}{2}(\vect{\theta}-\vect{\widetilde{\theta}})^T\left(\vect{V}_i^T\vect{S}_i\vect{V}_i+\vect{D}_i\right)(\vect{\theta}-\vect{\widetilde{\theta}})+(\vect{\theta}-\vect{\widetilde{\theta}})^T\vect{\mu}_i\right).
\end{equation}
When $i=n$, denote $\vect{V},\vect{S},\vect{D},\vect{\mu},z=\vect{V}_n,\vect{S}_n,\vect{D}_n,\vect{\mu}_n,z_n$, the formula $\ref{LPFB_2}$ is equivalent to $\ref{LPFB_1}$ and we finish the proof. We successfully decompose $\vect{\Sigma}_i$ into lower-rank while keep the upper-bound at the same time. In the following part, we are going to show how to construct the matrix sequence $\{\vect{V}_i\}_{i=1}^n$, $\{\vect{S}_i\}_{i=1}^n$ and $\{\vect{D}_i\}_{i=1}^n$ satisfies the condition \ref{sigma_bound}. The proof is based on mathematical induction.

\noindent\begin{itemize}[leftmargin=0.3in]
\item From proof of Thm \ref{thm_PFB} in \cite{jebara2012majorization}, $z_0\to 0^+$, $\vect{\Sigma}_0=z_0\vect{I}$. Define $\vect{V}_0=\vect{0}$, $\vect{S}_0=\vect{0}$ and $\vect{D}_0=z_0\vect{I}$, it is obvious that $\vect{x}^T\vect{\Sigma}_0\vect{x}=\vect{x}^T(\vect{V}_0^T\vect{S}_0\vect{V}_0+\vect{D}_0)\vect{x}$ holds for all $\vect{x}\in\mathbb{R}^d$.

\item Assume $\vect{V}_{i-1}$, $\vect{S}_{i-1}$ and $\vect{D}_{i-1}$ satisfies condition (\ref{sigma_bound}), we are going to find $\vect{V}_i$, $\vect{S}_{i}$ and $\vect{D}_{i}$ satisfies this condition as well. From line 5 in algorithm \ref{alg_PFB}, $\vect{\Sigma}_i=\vect{\Sigma}_{i-1}+\vect{r}\vect{r}^T$, where $\vect{r}=\sqrt{\beta}\vect{l}$. Define the subspace constructed by row vectors of $\vect{V}_{i-1}$ as $\mathcal{A}=span\{\vect{V}_{i-1}(1,\cdot),...,\vect{V}_{i-1}(k,\cdot)\}.$
Map the vector $\vect{r}$ onto subspace $\mathcal{A}$ and denote the residual orthogonal to subspace $\mathcal{A}$ as $\vect{g}$
\begin{align}\label{r_decompose}
    \vect{r}&=\sum_{j=1}^d\vect{r}\vect{V}_{i-1}(j,\cdot)\vect{V}^T_{i-1}(j,\cdot)+\vect{g}\notag\\
    &=\vect{V}_{i-1}^T\vect{V}_{i-1}\vect{r}+\vect{g},
\end{align}
substitute (\ref{r_decompose}) into $\vect{\Sigma}_i=\vect{\Sigma}_{i-1}+\vect{r}\vect{r}^T$, we have

\begin{align}\label{update_sigma}
\vect{\Sigma}_i
&=\vect{\Sigma}_{i-1}+(\vect{V}_{i-1}^T\vect{V}_{i-1}\vect{r}+\vect{g})(\vect{V}_{i-1}^T\vect{V}_{i-1}\vect{r}+\vect{g})^T \notag \\
&=\vect{V}_{i-1}^T(\vect{S}_{i-1}+\vect{V}_{i-1}\vect{r}\vect{r}_i^T\vect{V}_{i-1}^T)\vect{V}_{i-1}+\vect{D}_{i-1}+\vect{gg}^T+\vect{V}_{i-1}^T\vect{V}_{i-1}\vect{r}\vect{g}^T+\vect{g}\vect{r}_i^T\vect{V}_{i-1}^T\vect{V}_{i-1}.
\end{align}

Define $\vect{a}=\vect{V}_{i-1}^T\vect{V}_{i-1}\vect{r}$, Equation (\ref{update_sigma}) can be simplified as
\begin{align}
  \vect{\Sigma}_i=\vect{V}_{i-1}^T(\vect{S}_{i-1}+\vect{V}_{i-1}\vect{r}\vect{r}^T\vect{V}_{i-1}^T)\vect{V}_{i-1}+\vect{D}_{i-1}+\vect{gg}^T+\vect{ag}^T+\vect{ga}^T,
\end{align}
the corresponding quadratic form is 
\begin{align}\label{sigma_decompose}
  \vect{x}^T\vect{\Sigma}_i\vect{x}=\vect{x}^T\left[\vect{V}_{i-1}^T(\vect{S}_{i-1}+\vect{V}_{i-1}\vect{r}\vect{r}^T\vect{V}_{i-1}^T)\vect{V}_{i-1}+\vect{D}_{i-1}+\vect{gg}^T\right]+\vect{x}^T(\vect{ag}^T+\vect{ga}^T)\vect{x}.
\end{align}

By \textbf{Lemma} \ref{lemma3}, the maximum eigenvalue of matrix $\vect{ag}^T+\vect{ga}^T$ is $\|\vect{a}\|_2\|\vect{g}\|_2$. Define $\vect{P}=\|\vect{a}\|_2\|\vect{g}\|_2I$, we have

\begin{equation}\label{upperbound_g}
    \vect{x}^T(\vect{ag}^T+\vect{ga}^T)\vect{x}\leq\max\{\lambda|\lambda \text{ is eigenvalue of } A\}\vect{x}^T\vect{x}=\|\vect{a}\|_2\|\vect{g}\|_2\|\vect{x}\|_2^2=\vect{x}^T\vect{Px}.
\end{equation}

Define 
\begin{align}\label{upperbound_formula}
    \vect{D}'_{i-1}&=\vect{D}_{i-1}+\vect{P}\in\mathbb{R}^{d\times d} \text{ (diagonal),}\notag\\
\widetilde{\vect{\Sigma}_i}:&=\vect{V}_{i-1}^T(\vect{S}_{i-1}+\vect{V}_{i-1}\vect{r}_i\vect{r}_i^T\vect{V}_{i-1}^T)\vect{V}_{i-1}+\vect{D}'_{i-1}+\vect{gg}^T,
\end{align}
by (\ref{sigma_decompose}) and (\ref{upperbound_g}), it is easy to know,
\begin{align}
  \vect{x}^T\vect{\Sigma}_i\vect{x}\leq \vect{x}^T\widetilde{\vect{\Sigma}_i}x 
\end{align}
We have already prove $x^T\widetilde{\vect{\Sigma}_i}x$ is larger than $x^T\vect{\Sigma}_ix$, then all we need is the upper-bound of $x^T\widetilde{\vect{\Sigma}_i}x$. Perform SVD decomposition on $\vect{S}_{i-1}+\vect{V}_{i-1}\vect{rr}^T\vect{V}_{i-1}^T$ and denote the result as
$$\vect{Q}_{i-1}^T\vect{S}'_{i-1}\vect{Q}_{i-1}=svd(\vect{S}_{i-1}+\vect{V}_{i-1}\vect{rr}^T\vect{V}_{i-1}^T),$$
define $\vect{V}'_{i-1}=\vect{Q}_{i-1}\vect{V}_{i-1}$, then (\ref{upperbound_formula}) can be simplified as
\begin{align}
  \widetilde{\vect{\Sigma}_i}&=\vect{V}_{i-1}^T\vect{Q}_{i-1}^T\vect{S}'_{i-1}\vect{Q}_{i-1}\vect{V}_{i-1}+\vect{D}'_{i-1}+\vect{gg}^T\notag\\
  &=\vect{V}_{i-1}^{'T}\vect{S}'_{i-1}\vect{V}'_{i-1}+\vect{D}'_{i-1}+\vect{gg}^T\notag\\
  &=\underbrace{\left[\begin{matrix}\vect{V}_{i-1}^{'T}&\frac{\vect{g}}{\|\vect{g}\|_2}\end{matrix} \right]\left[\begin{matrix}\vect{S}'_{i-1}&\vect{0}^T\\\vect{0}&\|\vect{g}\|_2^2\end{matrix}\right]\left[\begin{matrix}\vect{V}'_{i-1}\\\frac{\vect{g}^T}{\|\vect{g}\|_2} \end{matrix}\right]}_{=:\vect{B}}+\vect{D}'_{i-1}.
\end{align}

It is easy to know that $rank(\vect{B})=k+1$ and $\vect{B}$ is a $(k+1)$-svd decomposition. In order to keep the construction of $\vect{\Sigma}_i$, we have no choice but to remove the smallest eigenvalue and corresponding eigenvector from matrix $\vect{B}$. 
\begin{itemize}
 \item[case 1)] $\|\vect{g}\|_2^2\leq\arg\min_{j=1...k}\vect{S}'_{i-1}(j,j)$, remove eigenvalue $\|\vect{g}\|_2^2$ and corresponding eigenvector $\frac{\vect{g}}{\|\vect{g}\|_2}$.
\begin{align}
    \vect{\Sigma}'_i=\vect{V}_{i-1}^{'T}\vect{S}'_{i-1}\vect{V}'_{i-1}\vect{D}'_{i-1}=\widetilde{\vect{\Sigma}_i}-\vect{gg}^t=\widetilde{\vect{\Sigma}_i}-c\vect{vv}^T,
\end{align}
 where $c=\|\vect{g}\|_2^2$,$\vect{v}=\frac{\vect{g}}{\|\vect{g}\|_2}$. In this case $\vect{V}_i=\vect{V}'_{i-1}$, $\vect{S}_i=\vect{S}'_{i-1}$.
 \item[case 2)]$\|\vect{g}\|_2^2>\arg\min_{j=1...k}\vect{S}'_{i-1}(j,j)$, remove $m^{th}$ (absolute value smallest) eigenvalue in $\vect{S}'_{i-1}$ and corresponding eigenvalue $\vect{V}'_{i-1}(m,\cdot)$.
\begin{align}
    \vect{\Sigma}'_i=\vect{V}_{i-1}^{'T}\vect{S}'_{i-1}\vect{V}'_{i-1}+\vect{D}'_{i-1}+\vect{gg}^t-\vect{S}'_{i-1}(m,m)\vect{V}_{i-1}(m,\cdot)\vect{V}_{i-1}(m,\cdot)^T=\widetilde{\vect{\Sigma}_i}-c\vect{vv}^T,
\end{align}

  where $c=\vect{S}'_{i-1}(m,m)$,$\vect{v}=\vect{V}_{i-1}(m,\cdot)$. In this case,
  
  \begin{align*}
    \vect{V}_i&=\left[\begin{matrix}\vect{V}'_{i-1}(1,\cdot)\\...\\\vect{V}'_{i-1}(m-1,\cdot)\\\frac{\vect{g}}{\|\vect{g}\|_2^2}\\...\\\vect{V}'_{i-1}(k,\cdot)\end{matrix}\right],\\
    \vect{S}_i&=diag\left(\vect{S}'_{i-1}(1,1),\cdots,\vect{S}'_{i-1}(m-1,m-1),\|\vect{g}\|_2^2,\cdots,\vect{S}'_{i-1}(k,k)\right).
  \end{align*}

\end{itemize}
In both cases
\begin{align}
    \vect{\Sigma}'_i&=\vect{V}_i^T\vect{S}_i\vect{V}_i+\vect{D}'_{i-1}\\
    \widetilde{\vect{\Sigma}_i}&=\vect{\Sigma}'_i+c\vect{vv}^T,
\end{align}
where $\vect{V}_{i}\in\mathbb{R}^{k\times d}$ orthnormal, $\vect{S}_{i}\in\mathbb{R}^{k\times k}$ and $\vect{D}'_{i-1}\in\mathbb{R}^{d\times d}$ are diagonal, $\vect{v}\in\mathbb{R}^d$ and $\vect{v}^T\vect{v}=1$.
Therefore, for all $\vect{x}\in\mathbb{R}^d$ we have
\begin{equation}
  \vect{x}^T\widetilde{\vect{\Sigma}_i}\vect{x}=\vect{x}^T\vect{\Sigma}'_i\vect{x}+c\vect{x}^T\vect{vv}^T\vect{x}=\vect{x}^T\vect{\Sigma}'_i\vect{x}+c(\vect{x}^T\vect{v})^2,
\end{equation}
$\vect{\Sigma}'_i$ is not sufficient for the condition $ \vect{x}^T\vect{\Sigma}'_i\vect{x}\geq \vect{x}^T\widetilde{\vect{\Sigma}_i}\vect{x}(\forall \vect{x})$, we would like to relax it by constructing 
 $\vect{\Sigma}''_i=\vect{\Sigma}'_i+\vect{F}$ $(\vect{F}\in\mathbb{R}^{d\times d}\text{ is diagonal})$
 such that $\vect{x}^T\vect{\Sigma}''_i\vect{x}\geq \vect{x}^T\widetilde{\vect{\Sigma}_i}\vect{x}(\forall \vect{x})$, which is equivalent to
 \begin{align}\label{F_1}
  \vect{x}^T\vect{\Sigma}''_i\vect{x}&\geq \vect{x}^T\widetilde{\vect{\Sigma}_i}\vect{x}\notag\\
  \vect{x}^T(\vect{\Sigma}'_i+\vect{F})\vect{x}&\geq \vect{x}^T(\vect{\Sigma}'_i+c\vect{vv}^T)\vect{x}\notag\\
  \sum_{i=1}^d\vect{x}^2(i)\vect{F}(i)&\geq c\left(\sum_{i=1}^d\vect{x}(i)\vect{v}(i)\right)^2.
 \end{align}
 Set $\vect{F}(i)=c|\vect{v}(i)|\sum_{j=1}^d|\vect{v}(j)|$, by Lemma \ref{lemma1} we have
 \begin{align}\label{F_2}
  \sum_{i=1}^d\vect{x}^2(i)\vect{F}(i)&=\sum_{i=1}^d|\vect{x}(i)|^2\vect{F}(i)\notag\\
  &\geq\left(\sum_{i=1}^d|\vect{x}(i)|\frac{\vect{F}(i)}{\sqrt{\sum_{j=1}^d\vect{F}(i)}}\right)\notag\\
  &=\left(\sum_{i=1}^d|\vect{x}(i)|\frac{c|\vect{v}(i)|\sum_{j=1}^d|\vect{v}(j)|}{\sqrt{c}\sum_{j=1}^dc|\vect{v}(i)|}\right)^2\notag\\
  &\geq c\left(\sum_{i=1}^d\vect{x}(i)\vect{v}(i)\right)^2
 \end{align}
 Combine (\ref{F_1}) and (\ref{F_2}), $\vect{\Sigma}''_i=\vect{\Sigma}'_i+\vect{F}$ is what we want. Let $\vect{D}_i\!=\!\vect{D}_{i-1}'\!+\!\vect{F}$, together with $\vect{V}_i$ and $\vect{S}_i$ defined in former case 1) and case 2), we have
 $$\vect{x}^T\vect{\Sigma}_i\vect{x}\leq\vect{x}^T\widetilde{\vect{\Sigma}}_i\vect{x}\leq\vect{x}^T\vect{\Sigma}''_i\vect{x}=\vect{x}^T\left(\vect{V}_i^T\vect{S}_i\vect{V}_i+\vect{D}_i\right)\vect{x}$$

\end{itemize}

\noindent We already construct sequences of matrix $\{\vect{V}_i\}_{i=1}^n$, $\{\vect{S}_i\}_{i=1}^n$ and $\{\vect{D}_i\}_{i=1}^n$ satisfies the condition (\ref{sigma_bound}). Define $\vect{V},\vect{S},\vect{D}=\vect{V}_n,\vect{S}_n,\vect{D}_n$, keep $\vect{\mu}$ and $z$ the same as what they are in Thm \ref{thm_PFB}, the formula (\ref{LPFB_2}) implies
\begin{equation}
  Z_{\vect{x}}(\vect{\theta})\leq z\exp\left(\frac{1}{2}(\vect{\theta}-\vect{\widetilde{\theta}})^T(\vect{V}^T\vect{S}\vect{V} + \vect{D})(\vect{\theta}-\vect{\widetilde{\theta}})+(\vect{\theta}-\vect{\widetilde{\theta}})^T\vect{\mu}\right)
  \end{equation}

\end{proof}

\end{document}